\documentclass{article}

\usepackage[main, final]{neurips_2026}

\usepackage{natbib} 

\usepackage[utf8]{inputenc} 
\usepackage[T1]{fontenc}    
\usepackage{hyperref}       
\usepackage{url}            
\usepackage{booktabs}       
\usepackage{amsfonts}       
\usepackage{nicefrac}       
\usepackage{microtype}      
\usepackage{xcolor}         

\usepackage{relsize}

\usepackage{amsthm}
\usepackage{bbm} 

\makeatletter
\newcommand*{\transpose}{%
  {\mathpalette\@transpose{}}%
}
\newcommand*{\@transpose}[2]{%
  \raisebox{\depth}{$\m@th#1\intercal$}%
}

\usepackage{amsfonts} 
\usepackage{amsmath} 
\usepackage{amssymb} 
\usepackage{mathtools} 
\usepackage{subcaption} 
\usepackage{wrapfig} 
\usepackage{nicefrac} 
\usepackage{multirow} 
\usepackage{arydshln} 

\usepackage{bm}

\def\eqref#1{equation~\ref{#1}}

\def\1{\bm{1}}

\def\rve{{\mathbf{e}}}

\def\rmO{{\mathbf{O}}}

\def\rmS{{\mathbf{S}}}
\def\rmT{{\mathbf{T}}}

\DeclareMathAlphabet{\mathsfit}{\encodingdefault}{\sfdefault}{m}{sl}
\SetMathAlphabet{\mathsfit}{bold}{\encodingdefault}{\sfdefault}{bx}{n}

\def\gF{{\mathcal{F}}}

\def\gN{{\mathcal{N}}}

\def\sN{{\mathbb{N}}}

\def\sR{{\mathbb{R}}}
\def\sS{{\mathbb{S}}}

\DeclareMathOperator{\Tr}{Tr}

\usepackage{enumitem}
\usepackage{graphicx} 
\usepackage[export]{adjustbox}
\usepackage{tikz-cd}
\usepackage{mathabx,epsfig}
\usepackage{quiver}
\usepackage{pgfplots}

\usepackage[dvipsnames]{xcolor}
\usepackage{colortbl}

\usepackage{algorithm} 
\usepackage{algpseudocode}
\usepackage{url}
\usepackage{adjustbox} 

\usepackage{hyperref}

\algblock{ParFor}{EndParFor}
\algnewcommand\algorithmicparfor{\textbf{parfor}}
\algnewcommand\algorithmicpardo{\textbf{do}}
\algnewcommand\algorithmicendparfor{\textbf{end\ parfor}}
\algrenewtext{ParFor}[1]{\algorithmicparfor\ #1\ \algorithmicpardo}
\algrenewtext{EndParFor}{\algorithmicendparfor}

\usepackage{booktabs}

\usepackage{natbib} 
\PassOptionsToPackage{sort}{natbib}
\definecolor{darkgreen}{rgb}{0.0, 0.5, 0.13}
\definecolor{bestgreen}{RGB}{108,192,146}

\usepackage{amsmath}
\usepackage{amssymb}
\usepackage{mathtools}
\usepackage{amsthm}
\usepackage{faktor}

\usepackage{float}

\usepackage[capitalize,noabbrev,nameinlink]{cleveref}

\theoremstyle{plain}
\newtheorem{theorem}{Theorem}[section]
\newtheorem{proposition}{Proposition}[section]
\newtheorem{lemma}{Lemma}[section]
\newtheorem{corollary}{Corollary}[section]
\newtheorem{example}{Example}[section]
\theoremstyle{definition}
\newtheorem{definition}{Definition}[section]

\theoremstyle{remark}
\newtheorem{remark}{Remark}[section]

\usepackage[textsize=tiny]{todonotes}

\usepackage{comment} 
\usepackage{utfsym}
\usepackage{caption}

\newcommand{\rmapf}[2]{\mathcal{F}_{#1 \unlhd #2}}

\newcommand{\crossmark}{\scalebox{0.75}{\usym{2613}}}

\usepackage[normalem]{ulem}

\definecolor{lb}{RGB}{31,119,180}
\usepackage{tcolorbox}
\newtcolorbox{mybox}[1]{colback=gray!15,colframe=black!20,fonttitle=\bfseries,title=#1, %
left=3pt,    
  right=3pt,   
}

\title{Mitigating Over-squashing without Rewiring: \\ A Sheaf Effective Resistance Perspective}

\author{%
  \hspace{-5pt}André Ribeiro \\
  \hspace{-5pt}Getulio Vargas Foundation\\
  \hspace{-5pt}\texttt{andre.guimaraes@fgv.br} \\
   \And 
   Germano Barcelos \\
   Federal University of Viçosa \\
   \texttt{gegenbarcelos@gmail.com} \\
      \AND
      \hspace{-5pt}Amauri H. Souza \\
   \hspace{-5pt}Federal Institute of Ceará, 2$\delta$ AI \\
   \hspace{-5pt}\texttt{amauriholanda@ifce.edu.br} \\
      \And
   \hspace{-40pt}Diego Mesquita \\
   \hspace{-40pt}Getulio Vargas Foundation, 2$\delta$ AI \\
   \hspace{-40pt}\texttt{diego.mesquita@fgv.br} 
   \And
   \hspace{-0pt}Ana Luiza Tenório \\
   \hspace{-0pt}Getulio Vargas Foundation, UFRRJ \\
   \hspace{-0pt}\texttt{anatenorio@ufrrj.br} \\   
}

\begin{document}

\maketitle

\begin{abstract}

Graph Neural Networks (GNNs) often struggle to capture long-range dependencies due to over-squashing — a phenomenon in which the repeated compression of node embeddings into finite-size messages causes representations to collapse. Over-squashing is most often diagnosed as a property of the graph topology, with effective resistance serving as a principled measure of the bottleneck. We provide a complementary view on the matter: building on cellular sheaves, we introduce sheaf effective resistance, a generalization of effective resistance that depends on the sheaf attached to the graph, and we prove that for flat vector bundles, the over-squashing sensitivity in the Jacobian sense is upper bounded by a quantity related to the sheaf effective resistance between the nodes. The bottleneck thus need not lie in the graph itself: it can be relocated, and reduced, by adjusting the sheaf. We instantiate this idea in FlatNSD, a simple message-passing variant of Neural Sheaf Diffusion, and show that it implicitly learns to modulate total sheaf effective resistance, performing well on benchmarks designed to stress over-squashing without altering the original graph topology.
\end{abstract}

\section{Introduction}

Graph neural networks (GNNs) have become the dominant approach for learning on relational data, with applications ranging from molecular chemistry \citep{duvenaud2015convolutional,gilmer2017neural} to recommendation \citep{ying2018graph} and physical simulation \citep{sanchez2020learning}. Their performance, however, is constrained by two well-documented pathologies: over-smoothing, where iterated message passing makes node embeddings nearly indistinguishable, and over-squashing, where information from exponentially many nodes is compressed into fixed-size vectors, hindering long-range dependencies. Recent work by \citet{bodnar2022neural} introduced Neural Sheaf Diffusion (NSD), a generalization of Graph Convolutional Networks \citep{kipf2016semi} that uses cellular sheaves to mitigate over-smoothing while accommodating heterophilic graphs.

This raises a natural question: \textbf{can sheaf-based models also mitigate over-squashing?} The role of sheaves in this setting has been explored by \citet{bamberger2024bundle}, who introduced flat vector bundles as a scalable parameterization of sheaves and proposed BuNN, a sheaf-inspired architecture that steps outside the message-passing paradigm to address over-squashing. Building on their flat-bundle formulation, we show that the message-passing paradigm itself is sufficient to alleviate over-squashing, and we introduce FlatNSD as a simple NSD variant to support this claim.

Our analysis builds on a tool from the rewiring literature: effective resistance. The prevailing view treats over-squashing as a topological bottleneck. Through a Jacobian sensitivity analysis, \citet{toppingunderstanding} linked the issue to edges of high negative curvature; \citet{black2023understanding} introduced effective resistance as a metric for over-squashing and proposed adding edges to reduce it; and \citet{pmlr-v202-di-giovanni23a} further showed that over-squashing primarily affects distant nodes. Together, these works established graph topology as the main driver of over-squashing. 

On the other hand, we have sheaves: a cellular sheaf $\gF$ on a fixed graph is mathematical structure that assigns a vector space, the \emph{stalk}, to each vertex and edge, together with linear \emph{restriction maps} for each incident vertex–edge pair. Whereas \citet{hansen2020sheaf} considered fixed restriction maps, \citet{bodnar2022neural} proposed learning them end-to-end, yielding a flexible model that adapts to the task. Information propagation in NSD is governed by the sheaf Laplacian, which generalizes the graph Laplacian and is determined by the restriction maps. This suggests a different perspective on over-squashing: the graph itself need not be modified; one can instead adjust the sheaf $\gF$ assigned to it. Intuitively, restriction maps modulate the communication channels between nodes, so an appropriate choice of $\gF$ can play the role traditionally assigned to edge rewiring. This marks a shift in perspective: \textbf{the bottleneck responsible for over-squashing need not lie in the graph, but in the sheaf attached to it}. Rather than asking \textit{which edges to add/remove}, we ask \textit{which restriction maps to learn}.\looseness=-1

To make this precise, we generalize the classical effective resistance and introduce \emph{sheaf effective resistance} $\operatorname{R}_{\gF}(u,v)$, which quantifies the ease of transferring information between nodes $u$ and $v$ on a fixed graph. Our results show that lower $\operatorname{R}_{\gF}(u,v)$ corresponds to less compression of messages (\cref{theo:Jacobian&resistance}). Empirically, evaluations on synthetic data reveal a strong inverse relationship between total sheaf effective resistance and predictive accuracy (\cref{fig:sprm-tnm}). When using SAGE to learn the restriction maps, FlatNSD produces Laplacians with low sheaf effective resistance and correspondingly higher accuracy (\cref{fig:tree-by-depth}). Importantly, as we analyze in \cref{app:oversq+oversm}, our modifications mitigate over-squashing without inducing over-smoothing — addressing a natural concern given the well-known tension between the two phenomena. Our theoretical results extend beyond FlatNSD to a broader class of MPNNs that we refer to as SNNs (see \cref{eq:MP}).

In summary, our main contributions are the following:
\begin{enumerate}[leftmargin=16pt, itemsep=1pt, nosep] 
    \item We derive closed-form equations for the sheaf effective resistance (\cref{sec:sheaves_eff}) and prove that the sheaf effective resistance for flat vector bundles upper bounds the over-squashing sensitivity (\cref{theo:Jacobian&resistance}). \looseness=-1
    \item Through synthetic experiments that control over-squashing, we support the theory observing that accuracy increases while sheaf effective resistance decreases (\cref{fig:tree-by-depth} and  \cref{fig:sprm-tnm}). \looseness-1
    \item Empirical evaluations on real-world tasks show that FlatNSD offers a robust alternative, performing competitively alongside established rewiring and graph transformer baselines.
    \item We identify key architectural choices driving SNN performance (i.e., the role of restriction-map parameterization and stalk dimension) providing guidance for future sheaf-based models.
\end{enumerate}

\section{Background}
In this section, we briefly review cellular sheaves and their use in GNNs. Moreover, we overview how effective resistance was previously studied to avoid over-squashing.

\textbf{Notation.} Let $G = (V, E)$ be a connected, undirected, unweighted graph with $n$ vertices and $m$ edges. Let $A$ be the adjacency matrix of $G$ and $D$  its degree matrix. The graph Laplacian of $G$ is $L_G = D - A$ and its normalized version by $\Delta_G = I - \hat{A}= D^{-1/2} L_G D^{-1/2}$, with $\hat{A} = D^{-1/2}AD^{-1/2}$. 


\textbf{Cellular Sheaves.} 
    A cellular sheaf $\mathcal{F}$ over a graph $G$ associates (i) vector spaces $\mathcal{F}(v)$ and  $\mathcal{F}(e)$ to each vertex $v \in V$ and each edge $e \in E$, respectively; (ii) linear maps $\mathcal{F}_{v \unlhd e} : \mathcal{F}(v) \to \mathcal{F}(e)$ to each incident vertex-edge pair $v \unlhd e$. The vector spaces are called \textbf{stalks}, and the linear maps are called \textbf{restriction maps}. For simplicity, it is standard to assume that all stalks are isomorphic to $\sR^d$. If $d=1$ and all restriction maps are the identity map, we have the so called \textbf{trivial sheaf}. 

In GNNs, diffusion is often governed by the graph Laplacian. In SNNs, the sheaf Laplacian plays an analogous role, but its definition involves several intermediate steps.

\textbf{Sheaf Laplacian.} We start by defining spaces of direct sums $ C^0(G, \gF) = \bigoplus_{v \in V} \gF(v) \mbox{ and }
        C^1(G, \gF) = \bigoplus_{e \in E} \gF(e)$  to consider a linear map $\delta : C^0(G,\gF) \to C^1(G,\gF)$ given by $(\delta \mathbf{X})_e = \rmapf{u}{e} \mathbf{x}_u - \rmapf{v}{e} \mathbf{x}_v$ for each $e \in E$, where an orientation $e = v \to u$ is chosen, and $\mathbf{x}_v$ refers to a vector in $\gF(v)$\textcolor{black}{, i.e. a row in \textbf{X}}. When the sheaf is trivial, this operator coincides with $B^\top$, where $B$ is the incidence matrix of $G$. Then, the sheaf Laplacian is defined by $\delta^\top\delta$, generalizing the Laplacian description  as $B B^\top $. Equivalently, we also have that the \textbf{sheaf Laplacian} of  $\mathcal{F}$  is the linear operator $L_{\mathcal{F}}: C^0(G, \mathcal{F}) \to C^0(G, \mathcal{F})$ given by  
$        L_{\mathcal{F}}(\mathbf{X})_v:=\sum_{v, u \unlhd e} \rmapf{v}{e}^{\top}\left(\rmapf{v}{e} \mathbf{x}_v-\rmapf{u}{e} \mathbf{x}_u\right),\forall v \in V.$

The sheaf Laplacian $L_{\mathcal{F}}$ is a positive semidefinite matrix with diagonal blocks $(L_{\gF})_{vv} = \sum_{v \unlhd e} \rmapf{v}{e}^{\top} \rmapf{v}{e}$ and non-diagonal blocks $(L_{\gF})_{vu} = (L_{\gF})_{vu}^\top=-\rmapf{v}{e}^{\top} \rmapf{u}{e}$. \textcolor{black}{It can also be written as $L_{\gF} = D_{\gF} - A_{\gF}$, where $D_{\gF}$ is the block-diagonal sheaf-degree matrix, and $A_{\gF}$ the sheaf-adjacency matrix.} Its normalized version is $ \Delta_{\gF} = D_{\gF}^{-\frac{1}{2}}L_{\gF}D_{\gF}^{-\frac{1}{2}}$. 

\textbf{Neural Sheaf Diffusion (NSD).} One of the first, and currently the most popular SNN model was introduced in \citet{bodnar2022neural} as follows:
\begin{equation}\label{NSD_model}
    \mathbf{H}_{t+1} = \mathbf{H}_t - \sigma(\Delta_{\gF(t)}(\mathbf{I} \otimes \mathbf{W}_{1,t})\mathbf{H}_t\mathbf{W}_{2,t}),
\end{equation}
where $\mathbf{H}_0 \in \sR^{nd \times h}$ is the matrix of initial node features projected into $h$ channels, $\mathbf{W}_{1,t}$ and $\mathbf{W}_{2,t}$ are weight square matrices of dimensions $d$ and $h$, respectively, $\otimes$ is the Kronecker product, and $\sigma$ is an activation. Furthermore, the Laplacian is learned locally through the graph features: each restriction map $F_{v \unlhd e}$ is learned using a function $\Phi: \mathbb{R}^{2d} \to \mathbb{R}^{d \times d}$, with $\gF_{v \unlhd e\coloneqq(v,u)} = \Phi(\mathbf{x}_v, \mathbf{x}_u)$.

\textbf{Classes of sheaves.} \citet{bodnar2022neural} studied three classes of sheaves based on the properties of their restriction maps: diagonal, orthogonal, and general (i.e, any invertible linear map). The best balance between scalability and empirical performance was achieved by orthogonal restriction maps. We call this class of sheaves \textbf{vector bundles}. More recently, further improving the computational complexity of SNNs, \citet{bamberger2024bundle} presented  
\textbf{flat vector bundles} as part of their Bundle Neural Networks (BuNNs) model. They consider the special case of flat vector bundles in which to every node $v$ is assigned an orthogonal map $\gF_v$, which replaces $\rmapf{v}{e}$ for any adjacent edge $e$. Instead of building the model directly over the sheaf Laplacian, BuNN uses the bundle heat kernel.


\textbf{Effective Resistance.} Originally defined by \citet{kirchhoff1847ueber}, effective resistance is the voltage difference between two points in an electrical network when a unit current is applied. In the realm of GNNs, it serves as a measure of proximity and connectivity within a graph, 
and it was recently incorporated to mitigate over-squashing \citep{black2023understanding}. \looseness=-1
Formally, the effective resistance between nodes $u$ and $v$ is defined by
$ \operatorname{R}(u,v) = (1_u - 1_v)^{\top}L^+_G(1_u - 1_v),$  where $1_v$ is an indicator vector, i.e.,  a binary vector of length $n$ and in this vector, only the entry corresponding to node $v$ is $1$, and $L^+$ is the pseudo-inverse of the graph Laplacian.

\textbf{Over-squashing.} The connection between over-squashing and effective resistance passes through the work from \citet{toppingunderstanding}, which proposed an upper bound on the $2$-norm of the Jacobian between node features, $\|\partial h^{(r)}_u / \partial x_v \|$, with $h^{(r)}_u$ and $x_v$ both being vectors. This Jacobian quantifies how the initial feature $x_v$ at node $v$ affects $h^{(r)}_u$ at node $u$ in the $r$-th GNN layer.  Then, \citet{black2023understanding} proved that if $\operatorname{R}(u,v)$ decreases, then the upper bound increases, allowing for an improved information flow between $u$ and $v$. More precisely, using $\phi_t$ and $\psi_t$ to denote, respectively, learnable update and message functions at layer $t$, they obtained:

\begin{theorem}[Theorem 3.3 \cite{black2023understanding}]\label{theo:black} Let $G$ be a non-bipartite graph, and let $u, v \in V$. Suppose 
$\|\nabla \phi_l\| \leq \alpha$ and $\max\{\|\nabla \psi_l\|, 1\} \leq \beta$. 
Define $d_{\min} = \min\{d_u, d_v\}$ and $d_{\max} = \max\{d_u, d_v\}$. 
Let $\max\{|\mu_2|,...,|\mu_n|\} \le \mu$. Then \looseness=-1
  $$\left\|\frac{\partial h_u^{(r)}}{\partial x_v}\right\| \leq (2\alpha\beta)^r \frac{d_{max}}{2}\left(\dfrac{2}{d_{min}}\left(r +1+ \frac{\mu^{r+1}}{1-\mu}\right) - \operatorname{R}(u,v) \right).$$ 
\end{theorem}
\vspace{-0.2cm}
\citet{black2023understanding} rewiring technique adds new edges between certain nodes, aiming to decrease the total effective resistance --- the sum of the effective resistances between all pairs of nodes. This brings significant improvements in empirical results when compared with GNNs with no rewiring. 

Rewiring is a standard method to address over-squashing by improving the connectivity of a graph and potentially easing structural bottlenecks. Its general goal is to create new connections that enable a more efficient and balanced flow of information within the network. We focus on the comparison with \citet{black2023understanding} because they also work with effective resistance but other strategies can be applied: for instance, spectral gap \citep{karhadkarfosr,banerjee2022oversquashing}, or different notions of curvature \citep{fesser2024mitigating,giraldo2023trade}.\looseness=-1 

\section{Sheaf Neural Networks via Message Passing}
We analyze over-squashing through the lens of the Jacobian, in the same fashion as \cref{theo:black}. Therefore, we need to describe the NSD model in the message passing framework. 

Recall that for a graph whose nodes $u$ have neighborhoods $N_u \!=\! \{v \!\in \! V \,|\,(v,u)\!\in \!E\}$ and features $\textbf{h}_u^{(0)}\!\! = \!\textbf{x}_u$, a message passing layer is defined as
$\textbf{h}^{(t+1)}_v \!= \!\phi_t\!\!\left(\textbf{h}^{(t)}_v\!\!, \bigoplus_{u\in N_v}\!\!\psi_t\!\left(\!\textbf{h}^{(t)}_u,\textbf{h}^{(t)}_v\!\right)\!\!\right)$, 
where $\psi_t$ is the \textbf{message function}, $\phi_t$ is the \textbf{update function}, and $\oplus$ is a permutation-invariant aggregation function.\looseness=-1

The NSD model described in \cref{NSD_model}
is a message-passing GNN where the message function is updated for each layer since it is given by $\psi_t(\textbf{h}_u,\textbf{h}_v) = D_v^{-\frac{1}{2}}\rmapf{v}{e}^{\top}\left(\rmapf{v}{e} D_v^{-\frac{1}{2}}\mathbf{h}_v-\rmapf{u}{e} D_u^{-\frac{1}{2}}\mathbf{h}_u\right)$, if we ignore the weights, and the aggregation is the sum. This simple observation provides significantly more efficient implementation and accommodates directed graph inputs, which the original code from \citet{bodnar2022neural} does not. Moreover, if $\gF$ is a flat vector bundle, then  $$D_v^{-\frac{1}{2}}\gF_v^{\top}\gF_u D_u^{-\frac{1}{2}} = \frac{1}{\sqrt{d_ud_v}}\gF_v^{\top}\gF_u = \hat{A}_{uv}\gF_v^{\top}\gF_u.$$ 
Thus, we can write the FlatNSD model, which is NSD restricted to flat vector bundles, as \looseness=-1
\begin{equation}\label{eq:MP}
   \textbf{h}^{(t+1)}_v = \phi_t\left(\textbf{h}^{(t)}_v, \sum\limits_{u\in N_v}\hat{A}_{uv}\psi_t\left(\textbf{h}^{(t)}_u\right)\right), 
\end{equation}
where $\psi_t(\textbf{h}_u) =  -\gF_u\textbf{h}_u$, up to weight matrices, $\hat{A}_{uv}$ is the normalized adjacency matrix of the graph, and $\gF_v^{\top}$ is applied by $\phi_t$, since $\gF_v$ is linear. Therefore, the FlatNSD model is a Message Passing Neural Network (MPNN) that fits in the framework explored by \citet{black2023understanding}, allowing us to apply some of their results. In the following sections, we show that restricting attention to flat vector bundles yields additional computational and theoretical benefits. \textcolor{black}{In \cref{app:sheaf_models}, we provide a summary with our main contributions, distinguishing what is novel from prior work. We also provide a comparison of FlatNSD with existing sheaf architectures to better highlight its role as a tool to empirically verify our theoretical claims rather than a novel model.}

\section{Sheaf effective resistance}\label{sec:sheaves_eff}

We introduce sheaf effective resistance (sheaf ER) and theoretically connect it with over-squashing. \looseness=-1

\begin{definition}
\label{def:sheaf_res}
    The sheaf effective resistance between nodes $u$ and $v$ is defined by \vspace{-0.1cm} $$\operatorname{R}_{\mathcal{F}}(u,v) = (\mathbbm{1}_u - \mathbbm{1}_v)^{\top}L_{\mathcal{F}}^+(\mathbbm{1}_u - \mathbbm{1}_v),$$
   where $\mathbbm{1}_v$ is a  vector of dimension $nd$ where the $d$ entries corresponding to node $v$ is $1$. $L_{\mathcal{F}}^+$ is the pseudo-inverse of the sheaf Laplacian.
\end{definition}
\vspace{-0.1cm} 
If $\mathcal{F}$ is the trivial sheaf, then $d=1$ and sheaf effective resistance coincides with effective resistance.\looseness=-1

\begin{example}\label{Ex:sheaves_reduce_resistance}
    Consider the graph with only two vertices $u$ and $v$, and an edge $e$ between them. \vspace{-0.2cm} 
    \begin{enumerate}
        \item If $\mathcal{F}$ is the trivial sheaf, then $R_{\mathcal{F}}(u,v) = R(u,v) = 1.$ \vspace{-0.35cm} 
        \item If $\gF$ is a sheaf such that all stalks are isomorphic to $\sR^2$, $\rmapf{u}{e} = \begin{pmatrix}
            \cos(x) & -\sin(x) \\ \sin(x) & \cos(x)
        \end{pmatrix}$, and 
        $\rmapf{v}{e} = \begin{pmatrix}
            \cos(y) & -\sin(y) \\ \sin(y) & \cos(y)
        \end{pmatrix}$, then $R_{\gF}(u,v) = \cos(x-y) + 1$.
    \end{enumerate}
\end{example}
In \cref{app:proofs}, we provide more details to perform the calculations above. Observe that \textbf{by simply changing the sheaf, without changing the graph, we can reduce or increase the effective resistance} depending only on the restriction maps. Larger worked examples can be found in \cref{app:examples} \looseness=-1

From now on, we focus on flat vector bundles, deriving formulas for their effective resistance and proving the relationship with over-squashing sensitivity. Importantly, this restriction is highly practical: prior work shows vector bundles optimally balance scalability and performance \citep{bodnar2022neural}, with the \textit{flat} constraint further enhancing efficiency \citep{bamberger2024bundle}. Additionally, our preliminary experiments with general sheaves revealed significant numerical instability. Thus, flat vector bundles emerge as the most robust and practical choice for sheaf-based networks.

First, observe that for flat vector bundles, the sheaf Laplacian can be written as $L_{\gF} = \rmO^{\top}(L_G \otimes I_d)\rmO$, where $\rmO = \operatorname{blockdiag}(\gF_1, \gF_2, \dots, \gF_n) \in \sR^{nd \times nd}$ \citep{bamberger2024bundle}. In the following, we consider a flat vector bundle $\gF$ over $G$ and $L_G$ the graph Laplacian. 
\begin{proposition}
\label{prop:sheaf_lap_pinv}
    The pseudo-inverse of $L_{\gF}$ is $L_{\gF}^+ = \rmO^{\top}(L_G^+ \otimes I_d)\rmO$.
\end{proposition}


Crucially, this provides the following formula to compute $R_{\gF}(u,v)$.

\begin{proposition}
\label{lemma:sheaf_res}
Let $\mathbf{1}_d \in \sR^d$ be a vector of ones, and $M_{uv} = \mathbf{1}_d^{\top}\gF_u^\top\gF_v\mathbf{1}_d$. Then
  $$\operatorname{R}_{\mathcal{F}}(u,v) = d((L_G^+)_{uu}+(L_G^+)_{vv}) - 2M_{uv}(L_G^+)_{uv}.$$
\end{proposition}

If $\gF$ is the trivial sheaf, we obtain $\operatorname{R}(u,v) = (L_G^+)_{uu}+(L_G^+)_{vv} - 2(L_G^+)_{uv}$. So we have generalized a classical result. Importantly, this also allows us to write $\operatorname{R}_{\mathcal{F}}(u,v)$ in terms of $\operatorname{R}(u,v)$:
\begin{equation}\label{eq:resistances}
   \operatorname{R}_{\mathcal{F}}(u,v) = d\operatorname{R}(u,v) + 2(d - M_{uv})(L_G^+)_{uv}.
\end{equation}
Note that computing $R_{\gF}$ requires only a single pseudo-inverse calculation, as the sheaf effective resistance for any flat vector bundle $\gF$ depends on the pseudo-inverse of the \textbf{graph} Laplacian, which remains unchanged since we do not rewire the graph.

\Cref{theo:Jacobian&resistance} below formalizes the relationship between sheaf effective resistance and over-squashing.

\begin{theorem}\label{theo:Jacobian&resistance} Let $G$ be a non-bipartite graph, and let $u, v \in V$. Suppose 
$\|\nabla \phi_l\| \leq \alpha$ and $\max\{\|\nabla \psi_t\|, 1\}| \leq \beta$. 
Define $d_{\min} = \min\{d_u, d_v\}$, $d_{\max} = \max\{d_u, d_v\}$, and fix $d$ the dimension of the stalks. 
Let $\max\{|\mu_2|,..., |\mu_n|\} \le \mu$. If $(L_G^+)_{uv} \geq 0$, then 
\vspace{-4pt}
\begin{equation*}
   \left\|\frac{\partial h_u^{(r)}}{\partial x_v}\right\| \leq (2\alpha\beta)^{r} \frac{d_{max}}{2}\Biggl(\frac{2}{d_{min}}\Biggr(r + 1 +\frac{\mu^{r+1}}{1-\mu}\Biggr)+4(L_G^+)_{uv} - \frac{\operatorname{R}_{\gF}(u,v)}{d} \Biggr).
\end{equation*}
If $(L_G^+)_{uv} < 0$, then
\begin{equation*}
     \left\|\frac{\partial h_u^{(r)}}{\partial x_v}\right\| \leq (2\alpha\beta)^{r} \frac{d_{max}}{2}\Biggl(\frac{2}{d_{min}}\Biggr(r + 1 +\frac{\mu^{r+1}}{1-\mu}\Biggr) - \frac{\operatorname{R}_{\gF}(u,v)}{d} \Biggr).
\end{equation*}

\end{theorem}
\vspace{-4pt}
\cref{theo:Jacobian&resistance} shows that the propagation of messages can be controlled by the sheaf effective resistance. If the sheaf is fixed (the absence of a sheaf is equivalent to fixing the trivial sheaf), we must alter the graph to decrease the effective resistance. However, with varying sheaves, we can keep the original graph and decrease the sheaf effective resistance. Moreover, this provides a perspective change: \textbf{effective resistance measures communication capacity of the sheaf associated to the graph rather than to the graph itself.} \looseness=-1

To align theory with experiments, we calculate total sheaf effective resistance $\operatorname{R}_{\gF}$, i.e., the sum of the sheaf effective resistance $\operatorname{R}_{\gF}(u,v)$ for all pairs of nodes.   In the following, $\operatorname{R}$ denotes the total (graph) effective resistance.
\begin{theorem}\label{theo:totalsheafresistance_flatbundle}
    Let $\langle \cdot, \cdot \rangle_F$ denote the Frobenius inner product and $M$ as in \cref{lemma:sheaf_res}. Then
    $$\operatorname{R}_{\mathcal{F}} = dn\operatorname{Tr}(L_G^+) - \langle M, L_G^+\rangle_F = d\operatorname{R} - \langle M, L_G^+\rangle_F.$$
\end{theorem}

Importantly, if $\gF$ is the trivial sheaf, we obtain $\operatorname{R} = n\operatorname{Tr}(L_G^+)$. This is another classical result, so our approach is well-founded. We use the above in our implementation to analyze the total sheaf effective resistance in our experiments, using $M = SS^{\top}$, where $S_u = \mathbf{1}_d^{\top}\gF_u^{\top}$. This allows us to calculate the Frobenius inner product using the cyclic property of the trace while avoiding building $M$. \textcolor{black}{Furthermore, we characterize when flat bundles strictly decrease sheaf effective resistance.
\begin{proposition}\label{prop:existence}
    Let $G$ be a connected graph with at least two nodes and $\gF$ a flat $O(d)$-bundle over $G$. Then $\frac{R_{\gF}}{d} \leq R$, with equality holding if and only if $\gF_v\mathbf{1}_d=c$ for some $c\in\sR^d$ and all $v\in V$. In particular, for every $d\geq 1$, there exists a flat $O(d)$-bundle $\gF$ such that $\frac{R_{\gF}}{d}<R$.
\end{proposition}
}

We proceed with a corollary of \cref{theo:Jacobian&resistance} regarding the total sheaf effective resistance.
\begin{corollary}\label{cor:totalsheaf}
     Let $G$ be a non-bipartite graph, and let $u, v \in V$. Suppose 
$\|\nabla \phi_l\| \leq \alpha$ and $\max\{\|\nabla \psi_t\|, 1\}| \leq \beta$. 
Let $d_{\min} = \min\{d_u, d_v\}$, $d_{\max} = \max\{d_u, d_v\}$, and $d$ be the minimum stalk dimension. \looseness=-1 
Let $\max\{|\mu_2|,..., |\mu_n|\} \le \mu$. Then
\vspace{-4pt}
\begin{equation*}
\begin{split}
\mathlarger{\mathlarger{\sum}}\limits_{u\neq v} \left\|\frac{\partial h_u^{(r)}}{\partial x_v}\right\| &\leq (2\alpha\beta)^{r} \frac{d_{max}}{2}\Biggl(\dfrac{n(n-1)}{d_{min}}\Biggr(r + 1 +\frac{\mu^{r+1}}{1-\mu}\Biggr)- \frac{\operatorname{R}_{\gF}}{d}\Biggr).
\end{split}
\end{equation*}
\end{corollary}
\vspace{-4pt}
Importantly, FlatNSD does not require computing the sheaf ER or any pseudo-inverse. Rather, we use the sheaf ER as an analytical tool to understand the performance of our FlatNSD model, since $\operatorname{R}_{\gF}(u,v)$ is related to $\left\|\nicefrac{\partial h_u^{(r)}}{\partial x_v}\right\|$, which measures sensitivity to over-squashing. \looseness=-1

\textcolor{black}{
In the previous results, we assumed non-bipartite graphs, similar to \cite{black2023understanding}. We now provide extensions relating graph ER to the Jacobian of MP-based GNNs as in \cref{eq:MP} to bipartite graphs, with natural extensions to sheaf ER. We begin by constructing a formula for the Laplacian pseudo-inverse on bipartite graphs.
\begin{lemma}\label{lemma:bipart_pseudo}
    Let $G$ be a connected bipartite graph. Then $\Delta_G^+ = \frac{1}{2} P_{-} + \sum_{k=0}^{\infty} (\hat{A}^k - P_+ - (-1)^kP_-)$, where $P_+$ and $P_-$ are the outer product of the eigenvectors with eigenvalues $1$ and $-1$, respectively.
\end{lemma}
The next step is to find a formula for the effective resistance of bipartite graphs. 
\begin{proposition}\label{prop:bipart_res}
    Let $G$ be a connected bipartite graph, $b_{uv} = \mathbbm{1}_u - \mathbbm{1}_v$, $z_{uv} = D^{-1/2}b_{uv}$, and $\gamma_{uv} = z_{uv}^{\top}P_-z_{uv}$. Then $$R(u,v) = \frac{\gamma_{uv}}{2} + \sum_{k=0}^\infty \left( \frac{(\hat{A}^k)_{uu}}{d_u} + \frac{(\hat{A}^k)_{vv}}{d_v} - \frac{2(\hat{A}^k)_{uv}}{\sqrt{d_ud_v}} - (-1)^k\gamma_{uv} \right)$$
\end{proposition}
Now, we proceed to relate the above formula for the effective resistance with the Jacobian bound.
\begin{theorem}\label{theo:bipart_bound}
    Let $G$ be a connected bipartite graph. Let $u, v \!\in\! V = U \!\sqcup\! W$, $\|\nabla \phi_l\| \!\leq\! \alpha$ and $\max\{\|\nabla \psi_l\|, 1\| \leq \beta$, $d_{\min} = \min\{d_u,d_v\}$, $d_{\max} = \max\{d_u,d_v\}$, and $\overline{\mu} = \max_{i\neq1} |\mu_i| < 1$. Then
    $$\left\|\frac{\partial h_u^{(r)}}{\partial x_v}\right\| \leq (2\alpha\beta)^r \frac{d_{\max}}{2}\left(\frac{2}{d_{\min}}\left(r + 1 + \frac{\overline{\mu}^{r+1}}{1 - \overline{\mu}}\right) + \frac{1}{2}\gamma_{uv} - R(u,v)\right).$$
\end{theorem}
The corollary for the total effective resistance follows from this theorem, where the additional term accounting for bipartiteness becomes $\frac{4|U||W|}{\operatorname{vol}(G)}$ (see \cref{remark:bipart_corollary}). \\
\begin{remark}[Geometric interpretation under isospectrality]
    A natural question arises from the decomposition $L_{\gF}=\rmO^{\top}(L_G\otimes I_d)\rmO$: how can FlatNSD mitigate over-squashing if its diffusion operator has essentially the same spectral properties as the graph Laplacian? The key lies in the rotations performed by the restriction maps. Define $\xi=\frac{\mathbf{1}_d}{\sqrt d}$ and $g_v=\gF_v\xi\in\sS^{d-1}$. Then
    \begin{equation*}
        M_{uv} = \mathbf{1}_d^{\top}\gF_u^{\top}\gF_v\mathbf{1}_d = d\xi^{\top}\gF_u^{\top}\gF_v\xi = dg_u^{\top}g_v = d \cos \theta_{uv},
    \end{equation*}
    where $\theta_{uv}$ is the angle between $g_u$ and $g_v$. Thus, \cref{lemma:sheaf_res} can be rewritten as
    $\operatorname{R}_{\mathcal F}(u,v)
    =d\big((L_G^+)_{uu}+(L_G^+)_{vv}-2\cos(\theta_{uv})(L_G^+)_{uv}\big).$
    When $g_u=g_v$, we recover the effective resistance of the trivial $d$-dimensional flat bundle. If \(g_u\perp g_v\), the cross-term disappears, reducing effective resistance relative to the aligned configuration when \((L_G^+)_{uv}<0\). In this case, \(g_u=-g_v\) gives the largest reduction. When \((L_G^+)_{uv}>0\), alignment minimizes the effective resistance. \looseness=-1
\end{remark}
}  
\vspace{-2pt}
\begin{mybox}{}
\vspace{-5pt}
    \textbf{Conclusion.} Sheaf effective resistance upper bounds over-squashing, as measured by the Jacobian sensitivity. This suggests that learned sheaves can enhance representational capacity compared to the trivial sheaves provided by the data.
    \vspace{-6pt}
\end{mybox}


\section{Synthetic Experiments}\label{sec:experiments}

In this section, we test FlatNSD on three different synthetic datasets designed to address over-squashing. Our goal is to endorse that a simple SNN such as FlatNSD can alleviate it. \textcolor{black}{We provide experiments on the short-range over-squashing regime \citep{mishayev2025short} in \cref{app:short_range_ovsq}}.

\begin{wrapfigure}[15]{r}{0.5\textwidth}
\vspace{-5pt}
\includegraphics[width=.98\linewidth]{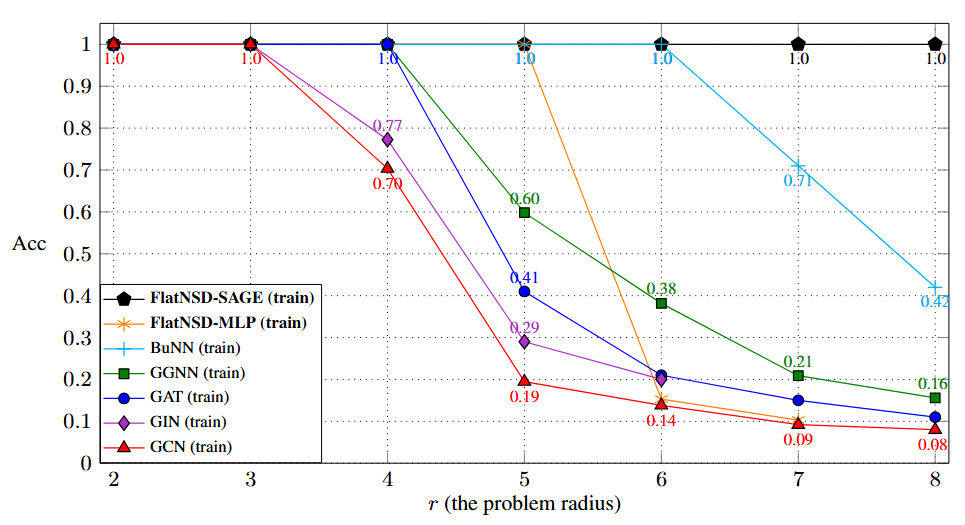}
\vspace*{-2mm}
\caption{Training accuracy for the TreeNeighborsMatch task. Results for GCN, GIN, GAT, and GGNN are from \cite{alon2020bottleneck}. Results for BuNN are from \cite{bamberger2024bundle}. Only FlatNSD-SAGE achieves perfect accuracy for all $r$. \looseness=-1}
\label{fig:tree-by-depth}
\vspace{-25pt}
\end{wrapfigure}
\textbf{TreeNeighborsMatch.} To test FlatNSD ability to deal with over-squashing, we reproduce the TreeNeighborsMatch problem proposed by \citet{alon2020bottleneck}, using the same framework. The TreeNeighborsMatch is a synthetic dataset aimed at evaluating a GNN’s ability to handle over-squashing, assessed through training performance. The underlying graph is a binary tree of depth $r$, and the task requires propagating information from the leaves to the root. Both the root and each leaf are attached to a set of \emph{blue neighbors}. The target label is determined by the unique leaf whose number of blue neighbors matches that of the root. We provide the parameters used for this task in \cref{tab:sweep-configs}.  
In \cref{fig:tree-by-depth}, we evaluate FlatNSD under two parameterizations of the restriction maps: an MLP (as in NSD) and SAGE (as in BuNN and CSNN \cite{ribeiro2025cooperative}). FlatNSD-SAGE outperforms all classical baselines (GCN, GIN, GAT, GGNN) as well as BuNN, matching the perfect accuracy of CSNN, a concurrent sheaf-based model. As a point of comparison, NSD with orthogonal maps attains $0.91$ accuracy at $r=3$ and $0.05$ at $r=4$. \textcolor{black}{To better understand how and why FlatNSD improves performance on this task, we provide an additional experiment with a controlled intervention on this task in \cref{app:intervention}}\looseness=-1

\begin{mybox}{}
\vspace{-6pt}
\textbf{Conclusion.} Restriction map parametrization plays a major role in NSD's performance and its ability to mitigate over-squashing. In particular, using flat bundles and structure-aware parameterizations (e.g., via GNNs) significantly boosts NSD's performance in long-range dependency tasks. \looseness=-1
\vspace{-7pt}
\end{mybox}
\begin{wrapfigure}[16]{r}{0.5\textwidth}
\vspace{-19pt}
\includegraphics[width=.98\linewidth]{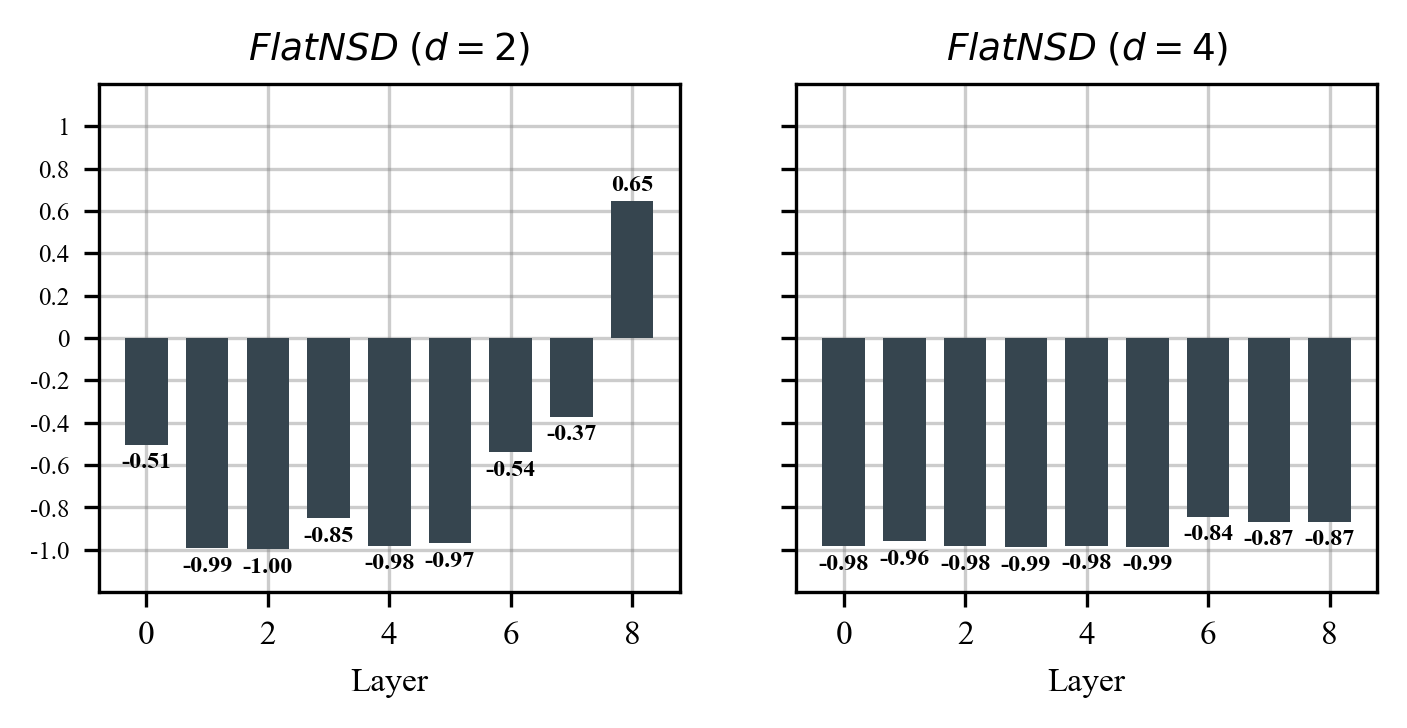}
\vspace*{-2.8mm}
    \caption{Correlation between Train Accuracy and Sheaf ER on FlatNSD-SAGE with $r=8$ for $d = 2$ and $d = 4$. The network struggles to find an optimal sheaf when $d=2$, leading to higher variance for the ER throughout the layers. For $d=4$, 
    the effective resistance is more uniform throughout the layers, while also achieving faster convergence.}
    \label{fig:sprm-tnm}
    \vspace{-25pt}
\end{wrapfigure}

\textbf{Effective Resistance.} By \cref{theo:Jacobian&resistance}, lower sheaf ER reduces sensitivity to over-squashing. Consequently, we expect a strong inverse relationship between accuracy and total ER in the TreeNeighborsMatch task, measured via Spearman correlation (\cref{fig:sprm-tnm}). For $r=8$, FlatNSD-SAGE reaches 100\% accuracy using both $d=2$ and $d=4$. However, $d=4$ converges much faster (25 steps vs. 143 for $d=2$) and maintains a strong inverse correlation across all layers. The weaker correlation for $d=2$ reflects its slower convergence, indicating the model struggles to approximate an optimal sheaf. Full results for all $r$ are detailed in \cref{app:corr}, together with Spearman correlation between Jacobian norm and total sheaf ER.

\begin{wraptable}[13]{r}{0.5\textwidth} 
\vspace{-1em} 
\centering
\scriptsize
\setlength{\tabcolsep}{3.5pt}
\captionsetup{width=6cm}
\caption{Performance (MAE) of FlatNSD. 
Darker green indicates lower error.\looseness=-1}
\label{tab:test_results_synt_compact}
\vspace{-7pt}
\resizebox{0.5\textwidth}{!}{
    \begin{tabular}{lccc}
    \toprule
    \textbf{Model} &
    \texttt{diam} $\,\downarrow$ &
    \texttt{ecc} $\,\downarrow$ &
    \texttt{sssp} $\,\downarrow$ \\
    \midrule
    A-DGN     
    & \cellcolor{bestgreen!70}$1.151_{\scriptstyle\,\pm\,0.038}$ 
    & \cellcolor{bestgreen!50}$4.981_{\scriptstyle\,\pm\,0.037}$ 
    & \cellcolor{bestgreen!50}$1.176_{\scriptstyle\,\pm\,0.140}$ \\

    DRew      
    & \cellcolor{bestgreen!60}$1.243_{\scriptstyle\,\pm\,0.047}$ 
    & \cellcolor{bestgreen!90}$4.651_{\scriptstyle\,\pm\,0.020}$ 
    & \cellcolor{bestgreen!40}$1.279_{\scriptstyle\,\pm\,0.011}$ \\

    GraphCON  
    & \cellcolor{bestgreen!10}$2.969_{\scriptstyle\,\pm\,0.189}$ 
    & \cellcolor{bestgreen!10}$5.474_{\scriptstyle\,\pm\,0.001}$  
    & \cellcolor{bestgreen!5}$5.734_{\scriptstyle\,\pm\,0.011}$ \\

    GCN       
    & \cellcolor{bestgreen!5}$3.832_{\scriptstyle\,\pm\,0.262}$  
    & \cellcolor{bestgreen!20}$5.233_{\scriptstyle\,\pm\,0.034}$ 
    & \cellcolor{bestgreen!20}$2.102_{\scriptstyle\,\pm\,0.094}$ \\

    GCNII     
    & \cellcolor{bestgreen!30}$2.005_{\scriptstyle\,\pm\,0.093}$ 
    & \cellcolor{bestgreen!10}$5.241_{\scriptstyle\,\pm\,0.030}$ 
    & \cellcolor{bestgreen!10}$2.128_{\scriptstyle\,\pm\,0.429}$ \\

    GIN       
    & \cellcolor{bestgreen!40}$1.630_{\scriptstyle\,\pm\,0.161}$ 
    & \cellcolor{bestgreen!60}$4.869_{\scriptstyle\,\pm\,0.092}$ 
    & \cellcolor{bestgreen!5}$2.234_{\scriptstyle\,\pm\,0.271}$ \\

    GPS       
    & \cellcolor{bestgreen!20}$2.160_{\scriptstyle\,\pm\,0.098}$ 
    & \cellcolor{bestgreen!80}${4.758}_{\scriptstyle\,\pm\,0.021}$ 
    & \cellcolor{bestgreen!70}$0.472_{\scriptstyle\,\pm\,0.050}$ \\

    {GRIT} 
    & \cellcolor{bestgreen!90}$1.014_{\scriptstyle\,\pm\,0.046}$ 
    & \cellcolor{bestgreen!30}${5.091}_{\scriptstyle\,\pm\,0.158}$ 
    & \cellcolor{bestgreen!90}$0.121_{\scriptstyle\,\pm\,0.013}$ \\

    PH-DGN    
    & \cellcolor{bestgreen!50}$1.627_{\scriptstyle\,\pm\,0.398}$ 
    & \cellcolor{bestgreen!40}$5.068_{\scriptstyle\,\pm\,0.126}$ 
    & \cellcolor{bestgreen!30}$1.323_{\scriptstyle\,\pm\,0.485}$ \\

    SWAN      
    & \cellcolor{bestgreen!90}$1.121_{\scriptstyle\,\pm\,0.070}$ 
    & \cellcolor{bestgreen!70}$4.840_{\scriptstyle\,\pm\,0.045}$ 
    & \cellcolor{bestgreen!60}$0.896_{\scriptstyle\,\pm\,0.232}$ \\

    CSNN	  
    & \cellcolor{bestgreen!100}$\mathbf{1.001}_{\scriptstyle\,\pm\,0.068}$ 
    & \cellcolor{bestgreen!90}$4.527_{\scriptstyle\,\pm\,0.022}$ 
    & \cellcolor{bestgreen!100}$\mathbf{0.044}_{\scriptstyle\,\pm\,0.036}$ \\

    FlatNSD$_{\text{SAGE}}$ 
    & \cellcolor{bestgreen!80}$1.132_{\scriptstyle\,\pm\,0.095}$ 
    & \cellcolor{bestgreen!100}$\textbf{3.138}_{\scriptstyle\,\pm\,0.052}$ 
    & \cellcolor{bestgreen!80}$0.198_{\scriptstyle\,\pm\,0.082}$ \\
    \bottomrule
    \end{tabular}
}
\end{wraptable}

We further substantiate that FlatNSD mitigates over-squashing by testing its performance on  ECHO-Synth \citep{miglior2025can} and Graph Transfer \citep{pmlr-v202-di-giovanni23a}, both capturing long-range interactions.

\textbf{ECHO-Synth.} The ECHO-Synth dataset challenges GNNs with long-range tasks across six topological families (diameters 17–40). Since this setting induces different types of structural compression, it is ideal for evaluating FlatNSD's capacity to alleviate over-squashing. Following \cite{miglior2025can}, we benchmark FlatNSD against their original models\textcolor{black}{, with the addition of CSNN \citep{ribeiro2025cooperative} to assess the capacity of a recent, more engineered sheaf model}. We use the same experimental setup and hyperparameter selection routine (details in \Cref{app:exp_setup}).

\textbf{Results.} \Cref{tab:test_results_synt_compact} reports MAE across the three property prediction tasks. FlatNSD$_{\text{SAGE}}$ demonstrates reliable effectiveness in handling long-range dependencies, establishing an impressive performance on the \texttt{ecc} task, outperforming the next best baseline (CSNN) by over 30\%. Furthermore, it consistently ranks among the top performers on the remaining tasks, securing the third-best result on \texttt{sssp} and remaining highly competitive on \texttt{diam}. These results confirm FlatNSD's capacity to effectively alleviate over-squashing across diverse structural bottlenecks. \textcolor{black}{Furthermore, CSNN is also effective in these tasks, achieving the lowest MAE on \texttt{diam} and \texttt{sssp}, highlighting the effectiveness of a more intricate sheaf-based model regarding over-squashing mitigation.}

\begin{wrapfigure}[17]{r}{0.4\textwidth}
\vspace{-12pt}
\includegraphics[width=1\linewidth]{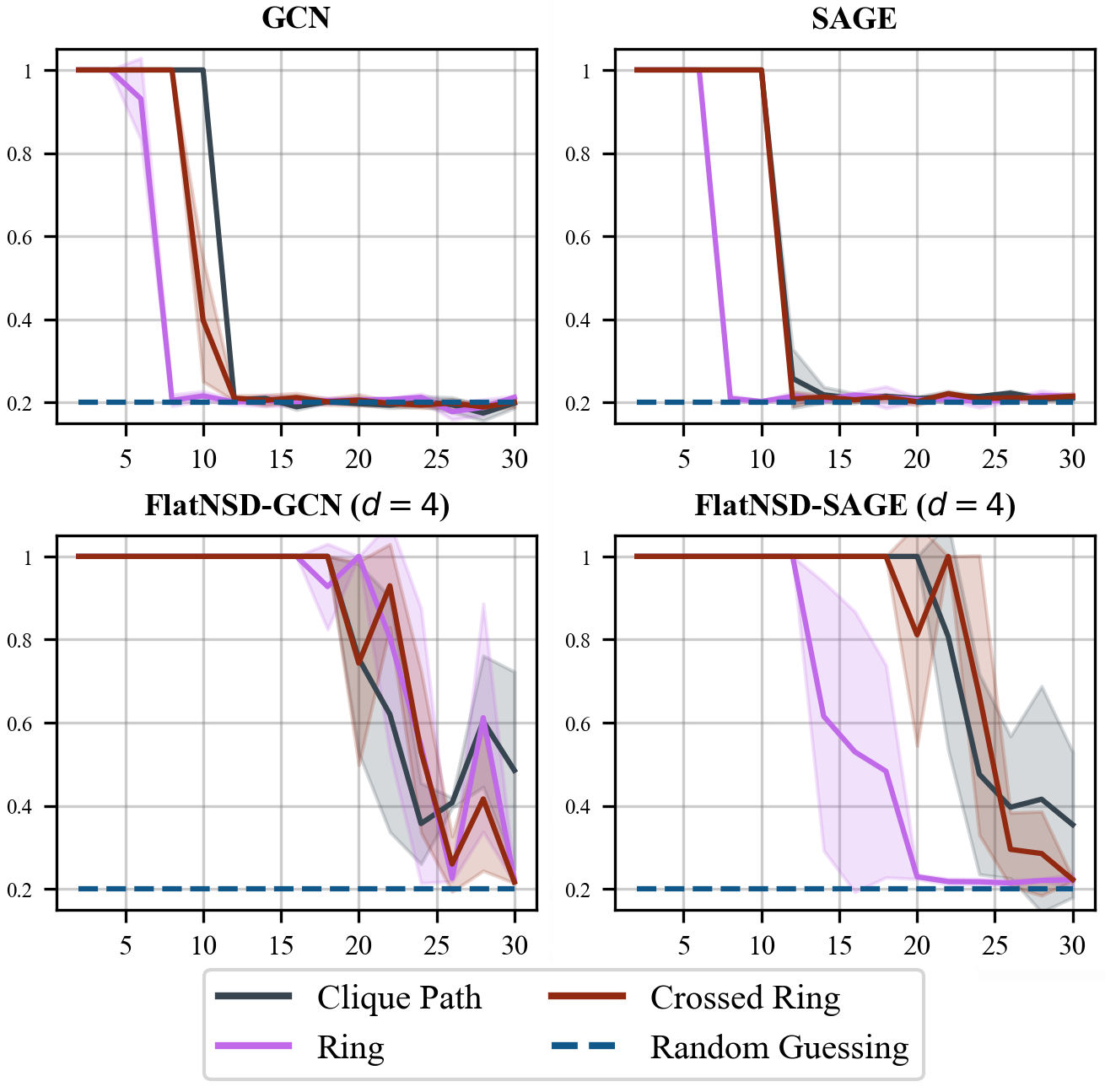}
\vspace*{-6mm}
    \caption{FlatNSD delays accuracy drop for all three topologies, demonstrating its ability of alleviating over-squashing. }
    \label{tab:graph_transferGCNandSAGE}
    \vspace{-25pt}
\end{wrapfigure}

\textbf{Graph Transfer.} We consider three graph topologies (CrossedRing, Ring, and CliquePath), each parameterized by the shortest-path distance $r$ between a designated source–target pair. The target node is assigned a one-hot label, and every other node receives a constant unitary feature. The model is trained to reproduce the target's label at the source. Since information must travel a distance $r$ through the graph, this task isolates over-squashing as the dominant failure mode, with accuracy expected to decrease as $r$ grows.

\textbf{Results.} \cref{tab:graph_transferGCNandSAGE} shows that with FlatNSD accuracy decreases after longer distances from source to target, in comparison with GCN and SAGE, for all three topologies. In \cref{app:FlatNSD_graph_transfer}, we show that FlatNSD also outperforms GIN and GAT, and discuss that for $d=4$ the performance is better than for $d=2$, suggesting the importance of stalk dimension on tasks sensitive to over-squashing, in alignment with our results for the TreeNeighborsMatch (\cref{fig:tree-by-depth}).

\vspace{-5.5pt}
\begin{mybox}{}
\vspace{-6.5pt}
\textbf{Summary.} FlatNSD was designed as a simple message-passing sheaf neural network that theoretically and empirically mitigates over-squashing across a diverse range of graph topologies. Rather than pursuing state-of-the-art performance, our overarching goal is to highlight SNNs as a fundamentally robust and promising direction for future architecture design.  \looseness=-1
\vspace{-7.5pt}
\end{mybox}

\section{Real-world experiments}

\begin{wraptable}[12]{r}{0.55\textwidth} 
\begin{minipage}{\linewidth}
    \centering
    \scriptsize
\scriptsize
\setlength{\tabcolsep}{3.5pt}
\vspace{-12pt}
\caption{Performance comparison on datasets from \cite{platonov2023critical}. AUROC is reported for minesweeper, tolokers and questions, accuracy is reported for the remaining datasets.}
\vspace{-6pt}
\label{table:nodeclass1}
\resizebox{\textwidth}{!}{
\begin{tabular}{lccccc}
\toprule
Model & roman-empire & amazon-ratings & minesweeper & tolokers & questions \\
Hom level & {0.05} & {0.38}  & {0.68} & {0.59} & {0.84} \\
\midrule
GCN & 73.69 ± 0.74 & 48.70 ± 0.63 & 89.75 ± 0.52 & 83.64 ± 0.67 & 76.09 ± 1.27 \\
SAGE & 85.74 ± 0.67 & \cellcolor{bestgreen!50}53.63 ± 0.39 & 93.51 ± 0.57 & 82.43 ± 0.44 & 76.44 ± 0.62 \\
GAT & 80.87 ± 0.30 & 49.09 ± 0.63 & 92.01 ± 0.68 & 83.70 ± 0.47 & 77.43 ± 1.20 \\
GAT-sep & 88.75 ± 0.41 & \cellcolor{bestgreen!25}52.70 ± 0.62 & 93.91 ± 0.35 & 83.78 ± 0.43 & 76.79 ± 0.71 \\
GT & 86.51 ± 0.73 & 51.17 ± 0.66 & 91.85 ± 0.76 & 83.23 ± 0.64 & 77.95 ± 0.68 \\
GT-sep & 87.32 ± 0.39 & 52.18 ± 0.80 & 92.29 ± 0.47 & 82.52 ± 0.92 & \cellcolor{bestgreen!25}78.05 ± 0.93 \\
\midrule
BuNN & \cellcolor{bestgreen!25}91.75 ± 0.39 & \cellcolor{bestgreen!95}\textbf{53.74 ± 0.51} & \cellcolor{bestgreen!50}98.99 ± 0.16 & \cellcolor{bestgreen!25}84.78 ± 0.80 & \cellcolor{bestgreen!50}78.75 ± 1.09 \\
CSNN & \cellcolor{bestgreen!95}\textbf{92.63 ± 0.50} & 52.07 ± 1.00 & \cellcolor{bestgreen!95}\textbf{99.07 ± 0.25} & \cellcolor{bestgreen!95}\textbf{85.45 ± 0.53} & \cellcolor{bestgreen!95}\textbf{79.31 ± 1.22} \\
\midrule
\textbf{FlatNSD}$_{\text{SAGE}}$ & \cellcolor{bestgreen!50}92.26 ± 0.46 & 50.51 ± 0.70 & \cellcolor{bestgreen!25}98.45 ± 0.34 & \cellcolor{bestgreen!50}84.90 ± 0.83 & 76.44 ± 1.25 \\
\bottomrule
\end{tabular}
}
\end{minipage}
\end{wraptable}

\textbf{Node Classification.} To further show that sheaf models are a robust alternative, we also evaluate FlatNSD on the five heterophilic graphs proposed by \citet{platonov2023critical}.  For binary classification
datasets, we report AUROC, while for multiclass datasets we report accuracy.

Following \citet{platonov2023critical}, who showed that classical GNNs often match or outperform architectures specifically designed for heterophilic graphs, we adopt their reference baselines: GCN \citep{kipf2016semi}, GraphSAGE \citep{hamilton2017inductive}, GAT \citep{velivckovic2017graph}, and GT \citep{shi2020masked}, together with the GAT-sep and GT-sep variants, which concatenate a node's representation with the mean of its neighbors rather than summing them \citep{zhu2020beyond}. We also compare FlatNSD against the recently proposed sheaf models BuNN \citep{bamberger2024bundle}, and CSNN \citep{ribeiro2025cooperative}. All three sheaf models considered use SAGE to learn the restriction maps.  \looseness=-1

\textbf{Result.} As shown in \cref{table:nodeclass1}, sheaf-based models are uniformly robust on heterophilic datasets. While more sophisticated variants surpass FlatNSD — BuNN by departing from the message-passing paradigm, and CSNN by coupling two cooperative sheaves to refine node communication within message passing — FlatNSD remains within close range of both. This supports our broader claim: the message-passing paradigm itself, when equipped with appropriate restriction maps, is already sufficient to alleviate over-squashing. 

\textbf{Graph classification.} We validate our theoretical claims through the Spearman correlation of sheaf effective resistance and accuracy on graph classification graphs Enzymes, IMDB, Mutag, and Proteins from the TUDatasets \cite{morris2020tudataset}. Our objective is to elucidate how cellular sheaves can alleviate over-squashing without changing the topology of the graph, but rather the structure of its underlying sheaf. In \cref{tab:performance_comparison_small_nodeclassification}  we conduct the same analysis for small graphs with a node classification task for Cora, Citeseer \cite{yang2016revisiting}, Texas, Cornell, Wisconsin \cite{pei2020geom}, and Chameleon \cite{rozemberczki2021multi}.\looseness=-1

\begin{wrapfigure}[37]{r}{0.6\textwidth} 
\begin{minipage}{\linewidth}
    \centering
\centering
\vspace{-12pt}
\scriptsize 
\setlength{\tabcolsep}{3.5pt} 
\caption{Performance of FlatNSD and GCN, GIN, and SAGE with different rewiring methods for both node and graph classification tasks. The method with the highest accuracy is highlighted in \textbf{bold}. DR requires graphs to have at least 3 nodes, but some graphs from Enzymes have only 2.}
\vspace{-7pt}
\label{tab:performance_comparison_gcn}
\begin{tabular}{lccccc}
\toprule
Method & Enzymes & IMDB & Mutag & Proteins & Rel Avg Imp \\
\midrule
GCN+NONE  & 25.5 ± 1.3 & 49.3 ± 1.0 & 68.8 ± 2.1 & 70.6 ± 1.0 & - \\
GIN+NONE  & 33.5 ± 1.3 & 67.7 ± 1.4 & 76.1 ± 3.1 & 69.5 ± 1.4 & - \\
SAGE+NONE & 31.5 ± 1.3 & 46.9 ± 0.8 & 72.2 ± 1.9 & 70.9 ± 0.9 & - \\
\midrule
GCN+SDRF  & 26.1 ± 1.1 & 49.1 ± 0.9 & 70.5 ± 2.1 & 71.4 ± 0.8 & +1.39\% \\
GIN+SDRF  & 32.4 ± 1.3 & 69.4 ± 1.4 & 79.5 ± 2.6 & 71.4 ± 0.8 & +1.61\% \\
SAGE+SDRF & 32.3 ± 1.5 & 47.1 ± 0.8 & 78.2 ± 1.7 & 69.9 ± 0.8 & +2.47\% \\
\midrule
GCN+FOSR  & 27.4 ± 1.1 & 49.6 ± 0.8 & 75.6 ± 1.7 & 72.3 ± 0.9 & +5.09\% \\
GIN+FOSR  & 28.8 ± 1.0 & \cellcolor{bestgreen!50}70.6 ± 1.3 & 74.8 ± 1.5 & 73.7 ± 0.8 & -1.35\% \\
SAGE+FOSR & 32.3 ± 1.5 & 47.1 ± 0.8 & 73.6 ± 2.0 & 69.9 ± 0.8 & +0.87\% \\
\midrule
GCN+BORF  & 24.7 ± 1.0 & 50.1 ± 0.9 & 75.8 ± 1.9 & 71.0 ± 0.8 & +2.31\% \\
GIN+BORF  & 31.4 ± 1.5 & \cellcolor{bestgreen!25}70.5 ± 1.3 & 78.2 ± 1.6 & 71.9 ± 1.3 & +1.02\% \\
SAGE+BORF & 26.5 ± 1.1 & 47.4 ± 0.7 & 69.8 ± 2.0 & 73.5 ± 0.8 & -3.62\% \\
\midrule
GCN+DR    & - & 47.0 ± 0.7 & 80.1 ± 1.8 & 72.2 ± 0.8 & +4.68\% \\
GIN+DR    & - & 64.8 ± 0.8 & 74.5 ± 2.8 & \cellcolor{bestgreen!95}\textbf{74.3 ± 0.8} & +0.17\% \\
SAGE+DR   & - & 47.3 ± 0.7 & 72.7 ± 2.0 & 70.5 ± 0.8 & +0.33\% \\
\midrule
GCN+GTR   & 27.4 ± 1.1 & 49.5 ± 1.0 & 78.9 ± 1.8 & 72.4 ± 1.2 & +6.27\% \\
GIN+GTR   & 28.4 ± 1.8 & 70.1 ± 1.2 & 78.5 ± 3.5 & 73.3 ± 0.9 & -0.76\% \\
SAGE+GTR  & 32.2 ± 1.1 & 47.4 ± 0.7 & 80.2 ± 1.6 & \cellcolor{bestgreen!50}74.2 ± 0.7 & +4.76\% \\
\midrule
GCN+LASER & 27.6 ± 1.3 & 50.3 ± 1.3 & 78.8 ± 1.6 & 71.8 ± 1.6 & +6.62\% \\
GIN+LASER & 35.3 ± 1.3 & 68.6 ± 1.2 & 76.1 ± 2.4 & 72.1 ± 0.7 & +2.61\% \\
SAGE+LASER& 31.9 ± 1.4 & 47.6 ± 0.7 & 75.5 ± 1.7 & 71.1 ± 0.8 & +1.90\% \\
\midrule
GCN+GOKU  & 27.6 ± 1.2 & 49.8 ± 0.7 & \cellcolor{bestgreen!50}81.0 ± 2.0 & 71.9 ± 0.8 & +7.21\% \\
GIN+GOKU  & 33.8 ± 1.2 & \cellcolor{bestgreen!95}\textbf{71.3 ± 0.9} & 78.4 ± 2.5 & \cellcolor{bestgreen!25}73.9 ± 1.0 & +3.89\% \\
SAGE+GOKU & 32.9 ± 1.2 & 47.2 ± 0.8 & 78.2 ± 1.7 & 72.8 ± 0.9 & +4.02\% \\
\midrule
FlatNSD$_{\text{GCN}}$ & \cellcolor{bestgreen!50}45.1 ± 1.2 & 65.7 ± 0.9 & 77.6 ± 1.8 & 72.2 ± 0.9 & \cellcolor{bestgreen!95}\textbf{+31.30\%} \\
FlatNSD$_{\text{GIN}}$ & \cellcolor{bestgreen!95}\textbf{47.4 ± 1.3} & 68.5 ± 0.9 & \cellcolor{bestgreen!95}\textbf{86.2 ± 1.6} & 71.5 ± 0.8 & \cellcolor{bestgreen!25}+14.71\% \\
FlatNSD$_{\text{SAGE}}$& \cellcolor{bestgreen!25}43.8 ± 1.5 & 66.5 ± 0.9 & \cellcolor{bestgreen!25}80.7 ± 1.7 & 70.6 ± 0.9 & \cellcolor{bestgreen!50}+23.05\% \\
\bottomrule
\end{tabular}
\end{minipage}

\vspace{-1pt} 

\begin{minipage}{\linewidth}
    \centering
\includegraphics[width=1\linewidth]{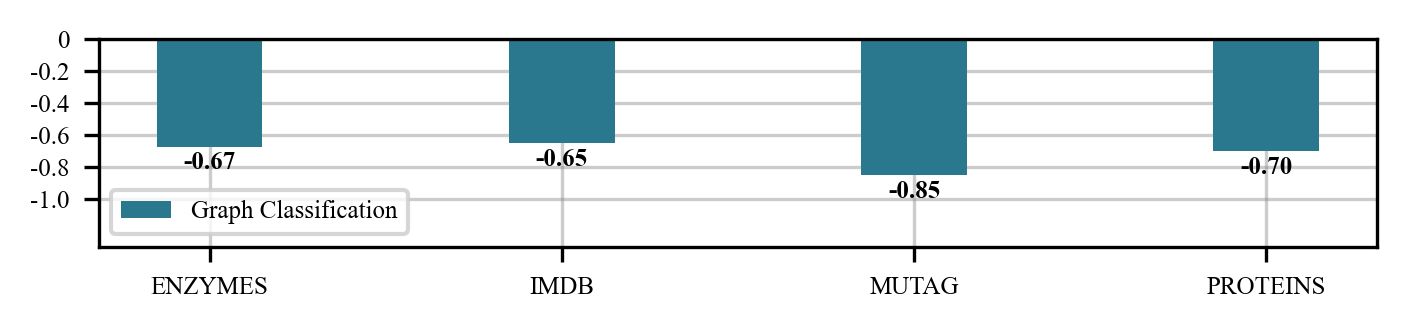}
\vspace{-20pt}
    \caption{Spearman Correlation between Sheaf Effective Resistance and Train Accuracy on FlatNSD-SAGE.} 
\label{fig:spearman_real_world}
\end{minipage}
\end{wrapfigure}

\textbf{Experimental setup.} 
Aiming  to prioritize fairness and understandability instead of obtaining the highest accuracy, we follow \citet{nguyen2023revisiting,liang2025mitigating}, that tune only rewiring-related hyperparameters, and freeze GNN-specific parameters and tune only the sheaf-related ones. We compare FlatNSD against GCN, GIN, SAGE, and  seven rewiring methods using these three GNNs as backbone models: SDRF \cite{toppingunderstanding},  FOSR \cite{karhadkarfosr}, BORF \cite{nguyen2023revisiting}, DR \cite{attali2024delaunay}, GTR \cite{black2023understanding}, LASER \cite{barberolocality}, and GOKU \cite{liang2025mitigating}. \Cref{app:exp_setup} provides details on our experimental setup. \looseness=-1

\textbf{Effective Resistance.} We use Spearman rank correlation coefficient to measure the correlation between accuracy and effective resistance. Since real-world data faces multiple challenges beyond requiring over-squashing mitigation, we do not expect the inverse relation to be as strong as in the controlled synthetic experiment reported in \cref{fig:sprm-tnm}. 
Nevertheless, the effective resistance is considerably correlated with training accuracy in these datasets, in alignment with the theory we developed. \looseness=-1

\textbf{Results.} \Cref{tab:performance_comparison_gcn} demonstrates that FlatNSD is highly robust: while not always the top performer on individual datasets, it presents a substantially higher relative average improvement. Crucially, FlatNSD robustness may be explained by its  
general efficacy outside of purely over-squashing-dominated regimes. We highlight this perspective, recalling the discussion on \cref{app:oversq+oversm} about the existence of sheaves that mitigate over-smoothing and over-squashing simultaneously. \looseness=-1

\section{Related work}
\textbf{Sheaf Neural Networks.} Building on the contributions of \citet{bodnar2022neural}, who demonstrated the advantages of the NSD model in handling heterophilic graphs and mitigating over-smoothing, we investigated the over-squashing phenomenon within SNNs. 
\citet{bamberger2024bundle} introduced flat vector bundles together with BuNN, an architecture that departs from the message-passing paradigm to address over-squashing. Our results suggest that the message-passing paradigm itself, when equipped with flat bundles, is already a strong contender: FlatNSD performs competitively with BuNN on the TreeNeighborsMatch task (\cref{fig:tree-by-depth}) and on node classification with larger graphs (\cref{table:nodeclass1}). Notably, much of the variation across sheaf-based models seems to be explained by the choice of restriction-map parameterization rather than by the underlying architecture (\cref{fig:tree-by-depth}). Direct head-to-head comparisons remain limited as the BuNN code was not publicly available at the time of submission. Crucially,  \citet{ribeiro2025cooperative} introduced sheaves on directed graphs to allow a selective behavior of each node, creating the CSNN model, which also handles over-squashing.  Other recently proposed sheaf-based models \citep{fiorini2025sheaves, mule2025directional} also formulate directed sheaves but do not discuss data compression. 

\textbf{Graph Rewiring.} Graph rewiring is a popular strategy  to address over-squashing. Part of its success is theoretically explained by \citet{pmlr-v202-di-giovanni23a} and empirically validated by different authors, using different heuristics \cite{black2023understanding, toppingunderstanding, karhadkarfosr, barberolocality,liang2025mitigating}. Moreover, relational GNNs are tools to improve rewiring methods that we exclude from our experiments, similarly to \citet{nguyen2023revisiting, liang2025mitigating}. This choice intends to isolate the contribution of FlatNSD and to consider graph rewiring methods and SNNs without further performance-oriented modifications.


\textbf{Over-squashing.} There is a variety of strategies to overcome over-squashing and the related problem of long-range propagation. They include but do not restrict to:  DE-GNN approaches \citep{gravina2025oversquashing, gravinaanti}, exploiting port-Hamiltonian dynamics \citep{heilig2025port}, adapting messages in the information flow \citep{errica2025adaptive, pmlr-v235-finkelshtein24a}, and Graph Transformers \citep{shirzad2023exphormer, rampavsek2022recipe, kreuzer2021rethinking}. Our experiments establish SNNs as a principled member of this toolkit for mitigating over-squashing.

\section{Conclusions}\label{sec:conclusion}

In this work, we established a theoretical link between Sheaf Neural Networks and over-squashing mitigation through the novel concept of Sheaf Effective Resistance. We demonstrated that FlatNSD can implicitly optimize restriction maps to facilitate information flow, all without altering the underlying graph topology. Our empirical analysis reveals that FlatNSD is a robust model, validating the efficacy of SNNs across diverse structural regimes.

\textbf{Future works.} Over-squashing is a well-known issue in GNNs, explored theoretically and empirically through various lenses. Many approaches rely on concepts deeply tied to effective resistance, such as the spectral gap, the Cheeger constant, and commute time. Notably, \citet{toppingunderstanding} introduce an edge-based metric called Balanced Forman curvature, linking negatively curved edges to over-squashing and proving that a positive lower bound on curvature controls the spectral gap. We envision that defining a notion of curvature for graph sheaves could be intrinsically related to the sheaf spectral gap. This perspective further positions sheaves as a powerful tool to alleviate over-squashing: bottlenecks traditionally viewed as intrinsic to the graph structure might be better understood—and resolved—as properties of the associated sheaf.

\textbf{Limitations.} We recall that our goal with this work was to analyze how FlatNSD can deal with over-squashing. To this end, we keep the model as simple as possible. Our experiments and theoretical results contribute to support that sheaves, through the lens of sheaf effective resistance, provide a framework genuinely capable of mitigating over-squashing. Nevertheless, the FlatNSD model can be strengthened. For instance, \citet{bamberger2024bundle}  considered a parallel dimension for multiple sheaves instead of a single sheaf per layer, which led to improved long-range interactions; \citet{fiorini2025sheaves} generalizes the sheaf Laplacian even further with sheaves designed to explicitly
account for directed graphs; \citet{ribeiro2025cooperative} introduces a sheaf based model where each node can learn whether it propagates or receives messages from/to its neighbors in each layer. 


\textbf{Computational Complexity.} Sheaf Neural Networks main drawback is its computational cost, but we address that the FlatNSD model we used in this work presents significant improvements in comparison to the original NSD model (\cref{app:complexity}), showing positive perspectives in this realm.   Nevertheless, in \cref{sec:runtime} we provide the runtime of FlatNSD and GCN with different rewiring methods, where we concluded that FlatNSD is considerable faster than graph rewiring methods for the tested node classification tasks while keeps the same order of magnitude for the tested graph classification tasks. This shows that FlatNSD is a versatile and balanced model.

\textbf{Societal Impact.} We  foresee no immediate societal or ethical concerns arising from this work.\looseness=-1 

\section*{Acknowledgments} 

We acknowledge the support of the Fundação Carlos Chagas Filho de Amparo à Pesquisa do Estado do Rio de Janeiro (FAPERJ) (SEI-260003/020348/2025, SEI-260003/020694/2025) and the Conselho Nacional de Desenvolvimento Científico e Tecnológico (CNPq) (404336/2023-0, 305692/2025-9, 445170/2024-7).

\bibliographystyle{plainnat}
\bibliography{conference}

\appendix

\section{Additional Examples}\label{app:examples}

To give a better intuition of how the choice of sheaves affects effective resistance, we provide two larger examples, building on the graph effective resistance examples from \cite{black2023understanding}.
\vspace{-8pt}
\begin{figure}[h]
    \centering
   \includegraphics[width=0.8\linewidth]{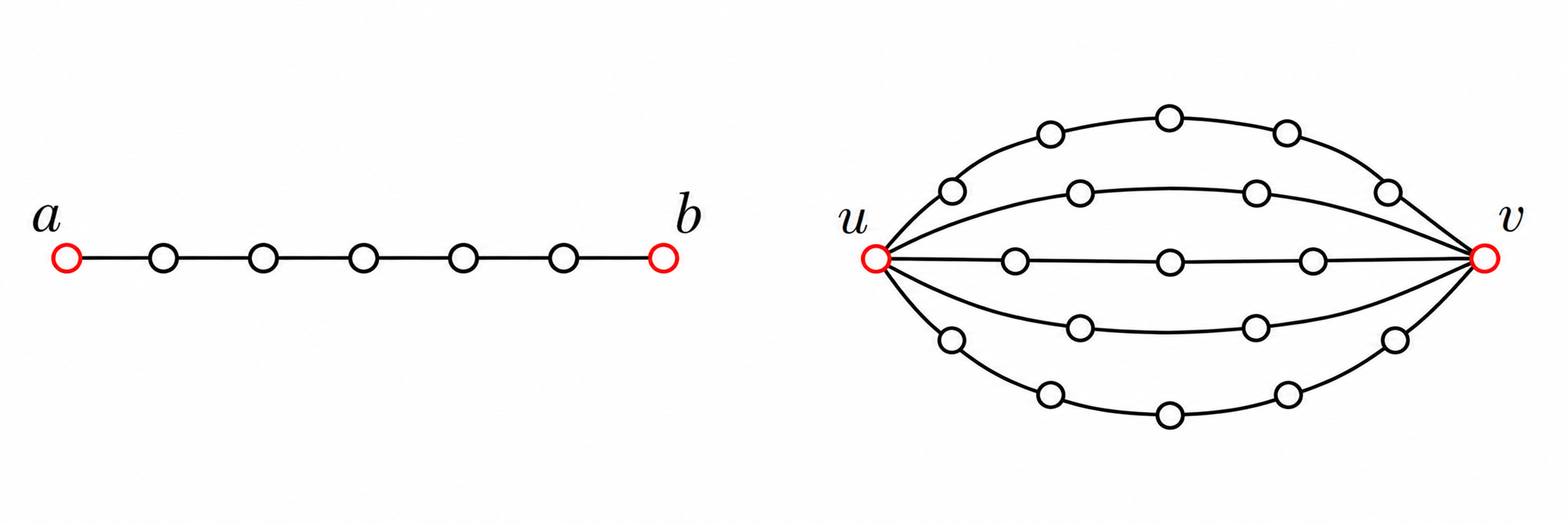}
       \vspace{-20pt}
    \caption{The two examples from \cite{black2023understanding}. For these graphs, we have $R(a,b) = 6$ (left) and $R(a,b) = \frac{10}{9}$ (right)}
    \label{fig:placeholder}
\end{figure}

Consider first the two graphs above: nodes $a,b$ connected by a path of length $6$, and nodes $a,b$ connected by four disjoint paths of lengths $5,4,4,5$. Their graph effective resistances are $R(a,b)=6$ and $R(a,b)=\frac{10}{9}$, respectively. From their Laplacian pseudo-inverses, we obtain $(L_G^+)_{ab}=-\frac{8}{7}$ and $(L_G^+)_{ab}=-\frac{185}{1152}$.

Let $Q(k)=\begin{pmatrix}\cos k & -\sin k\\ \sin k & \cos k\end{pmatrix}$, and define $\gF_a=Q(x)$, $\gF_b=Q(y)$, with identity maps on all remaining nodes. Using \cref{lemma:sheaf_res}, we obtain $$R_{\gF}(a,b)=\frac{52+32\cos(y-x)}{7} \quad \text{and} \quad R_{\gF}(a,b)=\frac{455+185\cos(y-x)}{288},$$ for the first and second graphs, respectively. Hence, for $y-x\in\{0,\frac{\pi}{2},\pi\}$, we have $$R_{\gF}(a,b)\in\left\{12,\frac{52}{7},\frac{20}{7}\right\} \quad \text{and} \quad R_{\gF}(a,b)\in\left\{\frac{20}{9},\frac{455}{288},\frac{15}{16}\right\},$$ for first and second graphs, respectively. Since $(L_G^+)_{ab}<0$ in both cases, rotations can only lower the endpoint resistance relative to the aligned configuration.

Importantly, however, the effect of the sheaf need not be uniform throughout the graph. Let $w$ be the first internal node adjacent to $a$ in the path graph. In this case, $(L_G^+)_{aw}=1$ and $R(a,w)=1$, yielding $R_{\gF}(a,w)=6-4\cos x$. Thus, every non-zero rotation at $a$ increases the resistance between $a$ and its neighbor. For example, taking $\gF_a=-I$ and all other maps equal to the identity gives $R_{\gF}(a,b)=\frac{20}{7}$ and $R_{\gF}(a,w)=10$, in contrast to $R_{\gF}(a,b)=12$ and $R_{\gF}(a,w)=2$ for the trivial sheaf. Therefore, a nontrivial sheaf can simultaneously lower the resistance between a distant pair while increasing it between nearby nodes. This is consistent with the role of the restriction maps: information can be transmitted efficiently from $a$ to $b$ even when it cannot be directly decoded by intermediate nodes along the communication path.

\section{Additional Experiments}

\subsection{Rewiring against FlatNSD on node classification}

We additionally test the performance of FlatNSD against rewiring methods on the node classification benchmarks used in \citet{nguyen2023revisiting,liang2025mitigating}. It achieves the best performance in 4 out of 6 datasets, and a generally higher overall improvement of the backbone GNN compared to baselines. We notice that Cora and Citeseer are homophilic graphs, an environment where it is already known that sheaf-based models do not flourish \cite{bodnar2022neural}. In contrast, graph rewiring techniques often struggle to improve their backbone GNNs on node classification tasks. This relates to the fact that graph rewiring methods primarily target over-squashing, whereas FlatNSD is more versatile.\looseness=-1

\begin{table*}[!htb]
\centering
\scriptsize 
\setlength{\tabcolsep}{6.0pt} 
\caption{Performance of FlatNSD and GCN, GIN, and SAGE with different rewiring methods for both node and graph classification tasks. The method with the higher accuracy is highlighted in \textbf{bold}. DR requires graphs to have at least 3 nodes, but some graphs from Enzymes have only 2.}
\vspace{-2pt}
\label{tab:performance_comparison_small_nodeclassification}
\begin{tabular}{lccccccc}
\toprule
Method & Cora & Citeseer & Texas & Cornell & Wisconsin & Chameleon & Rel Avg Improv \\
\midrule
GCN+NONE  & 86.7 ± 0.3 & 72.3 ± 0.3 & 44.2 ± 1.5 & 41.5 ± 1.8 & 44.6 ± 1.4 & 59.2 ± 0.6 & - \\
GIN+NONE  & 77.5 ± 1.2 & 59.3 ± 1.3 & 46.4 ± 4.5 & 40.2 ± 4.8 & 42.1 ± 3.6 & 56.6 ± 0.8 & - \\
SAGE+NONE & 87.1 ± 0.1 & 75.0 ± 0.1 & 82.4 ± 0.3 & 73.2 ± 0.3 & 77.3 ± 0.9 & 66.1 ± 0.9 & - \\
\midrule
GCN+SDRF  & 86.3 ± 0.3 & 72.6 ± 0.3 & 43.9 ± 1.6 & 42.2 ± 1.5 & 46.2 ± 1.2 & 59.4 ± 0.5 & +0.81\% \\
GIN+SDRF  & 77.2 ± 0.8 & 59.9 ± 1.2 & 44.5 ± 3.9 & 38.2 ± 4.1 & 43.1 ± 1.6 & 58.1 ± 1.4 & -0.57\% \\
SAGE+SDRF & 87.1 ± 0.7 & 74.1 ± 0.7 & 81.9 ± 2.9 & 72.4 ± 2.5 & 76.5 ± 1.8 & 65.2 ± 1.2 & -0.88\% \\
\midrule
GCN+FOSR  & 85.9 ± 0.3 & 72.3 ± 0.3 & 46.0 ± 1.6 & 40.2 ± 1.6 & 48.3 ± 1.3 & 59.3 ± 0.6 & +1.41\% \\
GIN+FOSR  & 78.2 ± 0.8 & 61.4 ± 1.5 & 43.7 ± 4.1 & 39.4 ± 3.8 & 43.7 ± 2.8 & 59.3 ± 0.6 & +0.87\% \\
SAGE+FOSR & 87.5 ± 0.6 & 74.7 ± 0.8 & 81.1 ± 2.0 & 71.1 ± 2.2 & 75.4 ± 1.1 & 66.2 ± 0.8 & -1.12\% \\
\midrule
GCN+BORF  & \textbf{87.5 ± 0.2} & 73.8 ± 0.2 & 49.4 ± 1.2 & 50.8 ± 1.1 & 50.3 ± 0.9 & 61.5 ± 0.4 & +8.97\% \\
GIN+BORF  & 77.6 ± 0.9 & 60.8 ± 0.2 & 49.9 ± 3.4 & 39.9 ± 3.9 & 46.7 ± 2.2 & 59.2 ± 0.4 & +4.16\% \\
SAGE+BORF & 87.3 ± 0.6 & 73.5 ± 0.6 & 73.5 ± 2.3 & 73.0 ± 3.1 & 78.2 ± 3.7 & 64.9 ± 0.7 & -2.25\% \\
\midrule
GCN+DR    & 78.4 ± 1.2 & 69.5 ± 1.6 & 67.8 ± 2.5 & 57.8 ± 1.9 & 62.8 ± 2.1 & 58.6 ± 0.8 & +19.84\% \\
GIN+DR    & 64.6 ± 1.6 & 50.8 ± 2.0 & 57.3 ± 2.8 & 50.1 ± 2.7 & 50.1 ± 3.0 & 60.2 ± 0.8 & +7.08\% \\
SAGE+DR   & 87.5 ± 0.4 & 74.2 ± 0.7 & 81.1 ± 3.7 & 76.2 ± 2.4 & 77.1 ± 2.9 & \textbf{66.8 ± 1.1} & +0.45\% \\
\midrule
GCN+GTR   & 87.3 ± 0.4 & 72.4 ± 0.3 & 45.9 ± 1.9 & 50.8 ± 1.6 & 46.7 ± 1.5 & 57.6 ± 0.8 & +4.85\% \\
GIN+GTR   & 78.6 ± 1.3 & 62.6 ± 1.9 & 49.5 ± 2.9 & 39.4 ± 2.3 & 44.2 ± 2.4 & 57.1 ± 1.2 & +2.92\% \\
SAGE+GTR  & 87.5 ± 0.6 & 74.3 ± 0.7 & 76.8 ± 4.3 & 71.1 ± 3.0 & 78.0 ± 2.5 & 65.6 ± 1.1 & -1.66\% \\
\midrule
GCN+LASER & 86.9 ± 1.1 & 72.6 ± 0.6 & 45.9 ± 2.6 & 42.7 ± 2.6 & 46.0 ± 2.6 & 43.5 ± 1.0 & -2.67\% \\
GIN+LASER & 79.1 ± 1.0 & 66.5 ± 1.3 & 46.5 ± 4.5 & 44.5 ± 3.8 & 46.1 ± 4.6 & 59.8 ± 2.2 & +6.71\% \\
SAGE+LASER& 84.9 ± 0.5 & 74.5 ± 0.4 & 81.3 ± 3.1 & 76.5 ± 3.1 & 79.8 ± 2.5 & 54.7 ± 0.9 & -2.34\% \\
\midrule
GCN+GOKU  & 86.8 ± 0.3 & 73.6 ± 0.2 & 72.4 ± 2.2 & 69.4 ± 2.1 & 68.8 ± 1.4 & 63.2 ± 0.4 & +32.33\% \\
GIN+GOKU  & 78.4 ± 0.5 & 63.6 ± 1.3 & 60.2 ± 2.6 & 49.5 ± 3.5 & 57.6 ± 3.1 & 62.1 ± 0.6 & +17.97\% \\
SAGE+GOKU & 87.1 ± 0.5 & 74.1 ± 0.3 & 81.4 ± 2.1 & 71.9 ± 4.1 & 80.2 ± 2.1 & 66.1 ± 0.6 & -0.07\% \\
\midrule
FlatNSD$_{\text{GCN}}$ & 76.7 ± 0.4 & 68.5 ± 0.3 & 86.3 ± 1.1 & 79.4 ± 1.3 & \textbf{86.2 ± 0.8} & 51.6 ± 0.5 & +41.70\% \\
FlatNSD$_{\text{GIN}}$ & 76.6 ± 0.3 & 70.0 ± 0.3 & 84.4 ± 1.3 & 78.1 ± 1.3 & \textbf{86.2 ± 1.0} & 52.5 ± 0.4 & +48.43\% \\
FlatNSD$_{\text{SAGE}}$ & 83.4 ± 0.3 & \textbf{76.0 ± 0.3} & \textbf{86.9 ± 1.0} & \textbf{80.5 ± 1.2} & 85.4 ± 0.9 & 52.5 ± 0.4 & +0.40\% \\
\bottomrule
\end{tabular}
\end{table*}
\vspace{-16pt}
\begin{figure}[!htb]
    \centering  \includegraphics[width=0.9\linewidth]{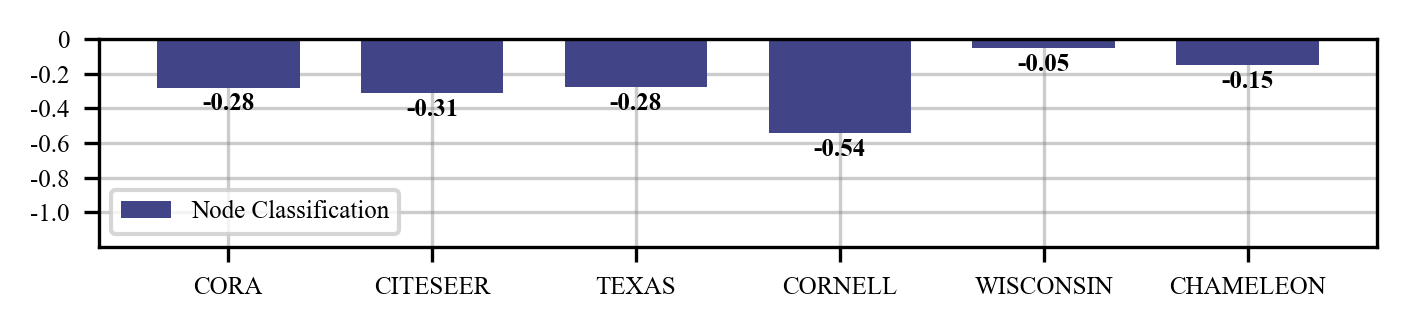}
    \vspace{-8pt}
    \caption{The figure shows the Spearman Correlation between Sheaf Effective Resistance and Train Accuracy on FlatNSD-SAGE. The correlation is weaker on these node classification tasks, aligning with the common understanding they are not a suitable setting for studying over-squashing.}    \label{fig:spearman_node_classification}
\end{figure}

\textbf{Results.}      
\Cref{tab:performance_comparison_small_nodeclassification} shows that FlatNSD is a robust model across both node and graph classification tasks (\cref{tab:performance_comparison_gcn}). 
The exceptions are Cora, Citeseer, and Chameleon. Cora and Citeseer are homophilic graphs, a regime where sheaf-based models are known to underperform their classical counterparts \citep{bodnar2022neural}; the drop relative to vanilla GCN on these datasets is therefore expected rather than surprising. Chameleon is known to contain duplicate nodes \cite{platonov2023critical}, which interfere with model performance. Overall, on the considered heterophilic datasets, FlatNSD  is competitive with or superior to graph rewiring.

\textbf{Effective Resistance.} The Spearman correlation shows a weak correlation for  these node classification.  This aligns with the view that the graph classification tasks performed in \cref{tab:performance_comparison_gcn} is more suitable for studying over-squashing: it requires aggregating node information into a graph-level embedding, thereby amplifying information compression.

Accordingly, \citet{saber2025empirical} suggest that rewiring strategies primarily pay off when over-squashing is severe and global, but the standard node classification benchmarks used here likely suffer only from mild compression.
Therefore, although this might attenuate the effectiveness of FlatNSD-SAGE regarding over-squashing mitigation on this set of problems, the model still achieves good performance in terms of raw accuracy, demonstrating its general efficacy outside of purely over-squashing-dominated regimes. We highlight this perspective, recalling the discussion on \cref{app:oversq+oversm} about the existence of sheaves that mitigate over-smoothing and over-squashing simultaneously. \looseness=-1

\subsection{Over-squashing and Over-smoothing}\label{app:oversq+oversm}

Most graph rewiring techniques address over-squashing by reducing effective resistance or, equivalently, enlarging the spectral gap and the Cheeger constant \cite{black2023understanding}. However, modifying the edge set of a graph also typically causes standard GNNs to yield increasingly similar node representations as the number of layers grows. In other words, mitigating over-squashing is often accompanied by a worsening of the over-smoothing effect. 

To overcome this, R-GNNs are used instead of vanilla GNNs \cite{karhadkarfosr,jamadandi2024spectral}. Rougly speaking, R-GNNs models can keep information from the original graph and use it through learned weights to keep balance between over-squashing and over-smoothing.


On the other hand, it is well known that NSD can circumvent over-smoothing. The question is: \emph{are sheaves able to simultaneously mitigate over-squashing and over-smoothing?}  \looseness=-1 

For simplicity, we consider a connected graph $G = (V,E)$ with $n$ vertices and two classes $A,B \subseteq V$. Let $\alpha_e$ for all $e \in E$ be arbitrary non-negative scalars. According to \citet[Proposition 10]{bodnar2022neural}, a sheaf whose stalks have dimension one equipped with invertible restriction defined by\looseness=-1  
\begin{equation}\label{eq:sheaf_ovsq&ovsm}
    \rmapf{v}{e} =  - \alpha_e \mbox{ if } v \in A \mbox{ and } \rmapf{u}{e} = \alpha_e \mbox{ if  } u \in B
\end{equation}
induces a diffusion that can linearly separate the classes of $G$ for almost all initial conditions. Therefore, even if we increase the number of layers, in such circumstances, the feature nodes will not converge to the same value.  \looseness=-1
 
We show that equipping a graph with the sheaves above can also reduce the total effective resistance, consequently reducing network sensitivity to over-squashing by \cref{theo:Jacobian&resistance}. \looseness=-1
\begin{theorem}\label{theo:resistance.vs.ovsm}
    Let $G$ be a connected graph with $n$ vertices and two classes $A,B \subseteq V$. If $\gF$ is a flat vector bundle as in \cref{eq:sheaf_ovsq&ovsm}, then $\operatorname{R}_{\gF} \leq \operatorname{R}$.
\end{theorem}
\vspace{-4pt}
\begin{proof}
    Since $d=1$, we only have to prove that $\langle M, L_G^+\rangle_F \geq 0$. 

   The flat and the orthogonal conditions force that $\alpha_e  = 1$ for every $e \in E$. So $\gF_v = -1$ if $v \in A$, $\gF_u = 1$, if $u \in B$, and then $M_{uv} = \pm 1$. Let $\beta = (\beta_1 , \dots, \beta_n)$, where $\beta_i = \gF_i$, for all $i \in V.$ Then $M = \beta\beta^\top$. Consequently,

   $$\langle M, L^+_G\rangle = \operatorname{Tr}(M^\top L^+_G) = \operatorname{Tr}(\beta\beta^\top L^+_G ) = \operatorname{Tr}(\beta^\top L^+_G\beta) = \beta^\top L^+_G\beta,$$

   where we use the trace property that $\operatorname{Tr}(AB) = \operatorname{Tr}(BA)$, with $A = \beta$ and $B = \beta^{\top} L^+_G$.

   Next, we consider the spectral decomposition $ L_G =  Q \Lambda Q^\top$ where $\Lambda = \operatorname{diag(0,\lambda_2,\dots,\lambda_n)}$ with $0 = \lambda_1 \leq \lambda_2 \leq \dots \leq \lambda_n$ eigenvalues of $L_G$. Then $\Lambda^+ = \operatorname{diag(0,\frac{1}{\lambda_2},\dots,\frac{1}{\lambda_n})}$ and $L^+_G =Q \Lambda^+ Q^\top$.

   Therefore, $$\langle M, L^+_G\rangle = \beta^\top L^+_G\beta = \sum\limits_{k=2}^{n}\frac{(q_k^\top \beta)^2}{\lambda_k} \geq 0.$$
\end{proof}

Therefore, if the number of classes of $G$ is $C = 2$, there is a sheaf that handles both over-smoothing and over-squashing. For graphs with $C > 2$ classes, \citet{bodnar2022neural} argued that sheaves whose stalks have dimension $d \geq C$ can mitigate over-smoothing. Our experiments in \cref{tab:performance_comparison_gcn} 
suggest that a higher dimensional version of the above theorem holds, since FlatNSD presents good performance on real-world datasets for node and graph classification tasks, where both over-smoothing and over-squashing may co-exist. 
\begin{mybox}{}
\vspace{-5pt}
    \textbf{Conclusion.} Sheaves can alleviate over-smoothing leveraging the existence of a non-trivial equilibrium space of the diffusion process \citep{bodnar2022neural}. Within the same class of sheaves, we can also mitigate over-squashing by lowering the total sheaf effective resistance.
    \vspace{-5pt}
\end{mybox}

To substantiate the theoretical insight, we include experiments demonstrating that FlatNSD can handle over-smoothing and over-squashing simultaneously.

\begin{wrapfigure}[13]{r}{0.4\textwidth}
\vspace{-9pt}
 \includegraphics[width=\linewidth]{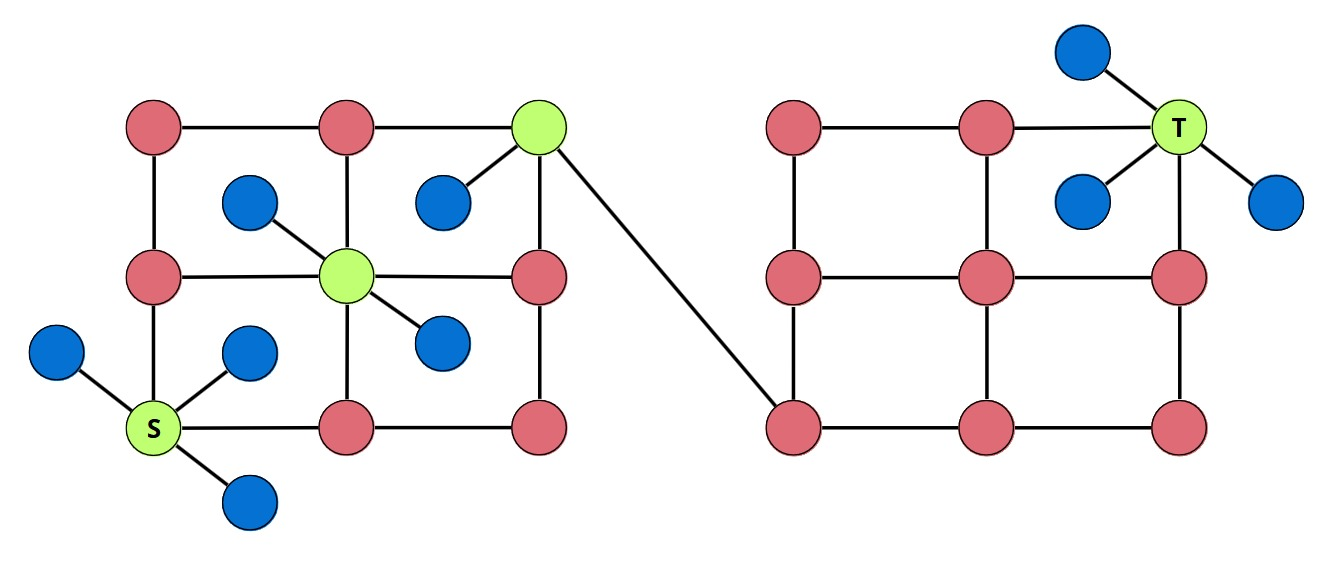}
    \caption{Example of graph in the GridTransfer task. Two grids are linked by an edge connecting two nodes in a diagonal. The task is to transfer information from one of the nodes in the diagonal of the node in the path to the most-distant node in the other grid.}
    \label{fig:placeholder}
\end{wrapfigure}

As explained by \citet{alon2020bottleneck}, there is no reason to expect over-smoothing to occur in the TreeNeighborsMatch problem. To the best of our knowledge, there is no synthetic benchmark specifically designed to evaluate the ability of GNNs to cope with both over-squashing and over-smoothing. We therefore propose GridTransfer as a task to study this setting.

\textbf{GridTransfer.} We construct a graph by taking two $n \times n$ grids and adding a single edge that connects one node on the outermost diagonal of the first grid to one node on the outermost diagonal of the second grid. The task is to transfer information from a \emph{target node} in the first grid to a node in the second grid located at the opposite corner relative to the endpoint of the inter-grid connection.

The nodes along the main diagonal of each grid act as \emph{green nodes}, as in the \textsc{NeighborsMatch} task, and each green node is connected to a number of \emph{blue neighbors}. The label is determined by matching the number of blue neighbors: the green node whose number of blue neighbors equals that of the target node contains its true label.

We adopt the same model configurations as in \textsc{TreeNeighborsMatch} and choose the network depth as $4(n-1)+1$, which corresponds to the worst-case receptive-field requirement for transferring information across the two-grid construction. Since this depth increases linearly with $n$, larger grids induce much deeper networks, where over-smoothing is expected to be more pronounced. Results are shown in \cref{fig:grid-transfer}. In comparison with \textsc{TreeNeighborsMatch}, we observe that FlatNSD-SAGE remains robust, indicating that over-smoothing has limited impact on its performance in this setting.

\begin{wrapfigure}[12]{r}{0.5\textwidth}
\vspace{-10pt}
    \centering
    \includegraphics[width=0.85\linewidth]{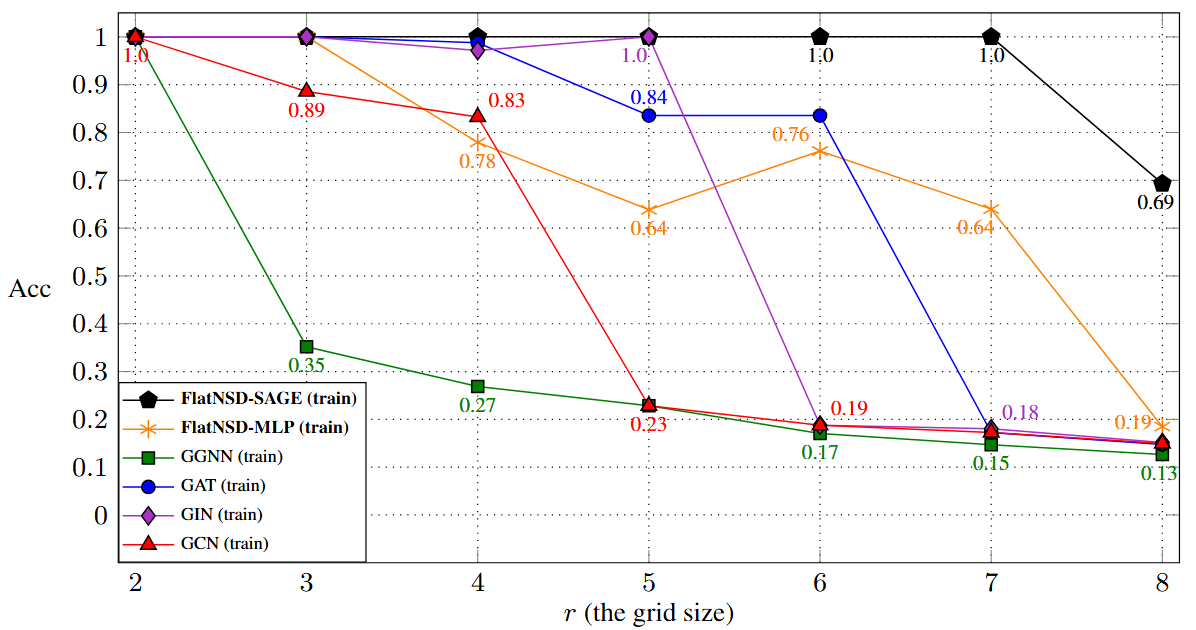}
    \caption{Accuracy of the models on the training set for the GridTransfer task. FlatNSD-SAGE is the only model to achieve  accuracy $1$ until $r=7$.}
    \label{fig:grid-transfer}
\end{wrapfigure}

\textbf{Dirichlet Energy.} \citet{cai2020note} used Dirichlet energy, given by the quadratic form $ E_{G}(\mathbf{X}) \coloneqq \operatorname{tr}(\mathbf{X}^{\top} \Delta_{G}\mathbf{X})$, where $\Delta_G$ is the normalized Laplacian of the graph $G$, as a measure of ``smoothness'' of node embeddings. 

From the perspective of Laplacian smoothing, deeper GCN propagation increasingly favors low-frequency signals, which generally lowers the Dirichlet energy and may cause node embeddings to become nearly constant on each connected component. Hence, mitigating over-smoothing should manifest as maintaining higher Dirichlet energy across layers, or equivalently, a less pronounced decrease.

In \cref{fig:DirichletEnergy}, we observe that FlatNSD with restriction maps learned with a GCN increases the energy for Enzymes, IMDB, MUTAG, and Proteins. 
We find consistent empirical evidence that the learned restriction maps are associated with decreasing sheaf effective resistance during training in regimes sensitive to over-squashing. While this does not establish causality, it supports the theoretical link between resistance and information propagation and provides a diagnostic for when FlatNSD alleviates over-squashing -- whilst still alleviating over-smoothing through the increase of the Dirichlet energy.

\begin{figure}[h]
    \centering
    \includegraphics[width=0.65\linewidth]{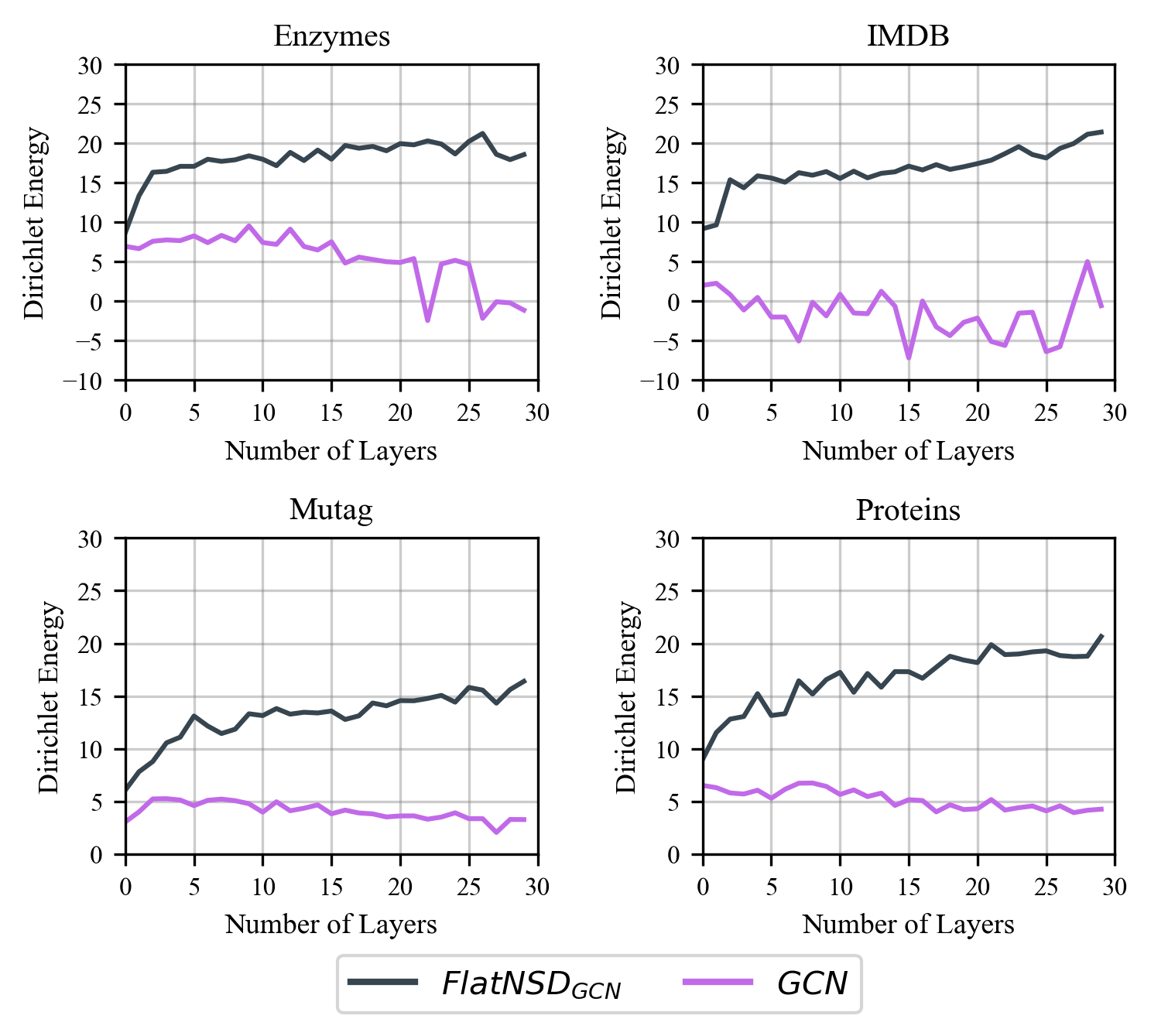}
    \caption{Log-Dirichlet Energy on the graph classification tasks of GCN and FlatNSD$_{\text{GCN}}$. The high Dirichlet Energy presented by FlatNSD-GCN against the low Dirichlet Energy presented by GCN is sign that FlatNSD alleviates over-smoothing even in datasets more strongly dominated by over-squashing.}
    \label{fig:DirichletEnergy}
\end{figure}

\subsection{Evaluating FlatNSD on Graph Transfer}\label{app:FlatNSD_graph_transfer}

The Graph Transfer task proposed by \citet{pmlr-v202-di-giovanni23a} as a benchmark for over-squashing, evaluates a model's ability to propagate information from a target node to a source node. The task is defined on three graph topologies: \textsc{Ring}, 
\textsc{CrossedRing}, and \textsc{CliquePath}, where $r$ denotes the distance between the source and the target node.

\begin{figure}[h]
    \centering
    \vspace{-10pt}
    \includegraphics[width=\linewidth]{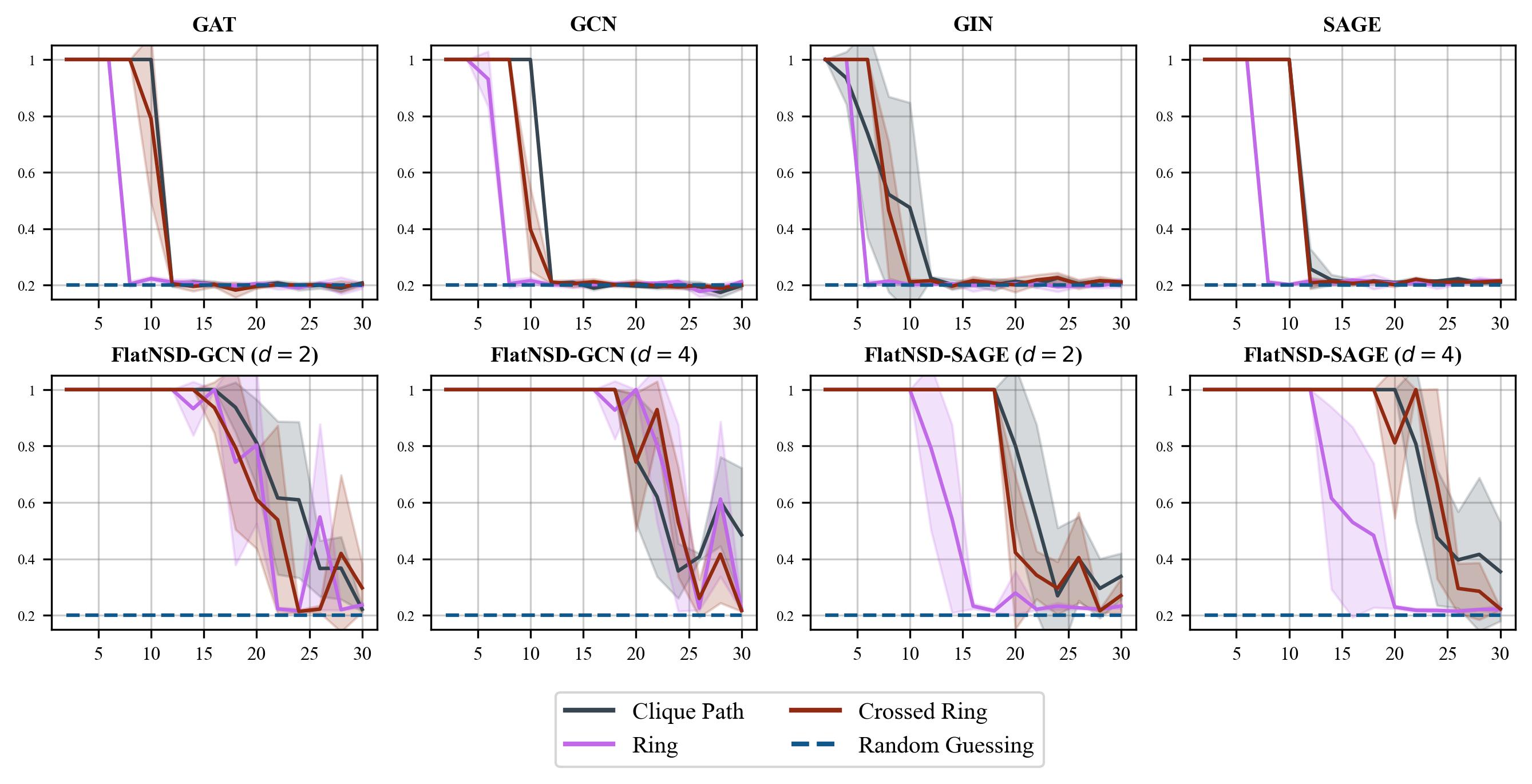}
    \vspace{-10pt}
    \caption{FlatNSD accuracy decreases only after longer distances from source to target, in comparison with GAT, GCN, GIN, and SAGE. Moreover, $d=4$ presents better performance than $d=2$, suggesting the importance of $d$ for tasks sensitive to over-squashing.}
    \label{fig:graphTransfer}
    \vspace{-10pt}
\end{figure}

\textsc{Ring} graphs are simple cycles of size $n$ in which  the source and target nodes lie at distance $\lfloor n/2 \rfloor$. \textsc{CrossedRing} graphs are cycles of the same size augmented with chords/crosses 
between auxiliary nodes; the added edges do not shorten the source–target distance, which remains $\lfloor n/2 \rfloor$. \textsc{CliquePath} graphs consist of a  $\lfloor n/2 \rfloor$-clique connected to a path of length $\lfloor n/2 \rfloor$, with the source placed inside the clique and the target at the far end of the path; the source–target distance is $\lfloor n/2 \rfloor + 1$. The three topologies are 
designed to isolate the effect of structural connectivity at a fixed distance: 
although $r$ is matched across topologies, over-squashing is most
severe for \textsc{CliquePath} and least severe for \textsc{CrossedRing}.


\textbf{Results.} For all three topologies, FlatNSD can postpone the accuracy drop, i.e., in FlatNSD accuracy decreases only after bigger distances from source to target, in comparison with all non-sheaf based models (GAT, GCN, GIN, SAGE), providing further evidence that FlatNSD mitigates over-squashing. Moreover, $d=4$ presents better performance than $d=2$, suggesting the importance of stalk dimension $d$ for tasks sensitive to over-squashing, in alignment with our results for the TreeNeighborsMatch task.

\subsection{Spearman Correlation on the TreeNeighborsMatch task}\label{app:corr}

Recall that Spearman correlation assesses the monotonic relationship between two variables and yields a value in 
$[-1,1]$.

\textbf{Accuracy.} We compute Spearman rank correlation per layer between training accuracy in the TreeNeighborsMatch and the sheaf effective resistance across all radii $r$  with $d=4$. While the model does not consistently recover the resistance-minimizing sheaf at each layer, it reliably does so in the initial layers.
\begin{figure}[h]
    \centering
    \vspace{-10pt}
    \includegraphics[width=\linewidth]{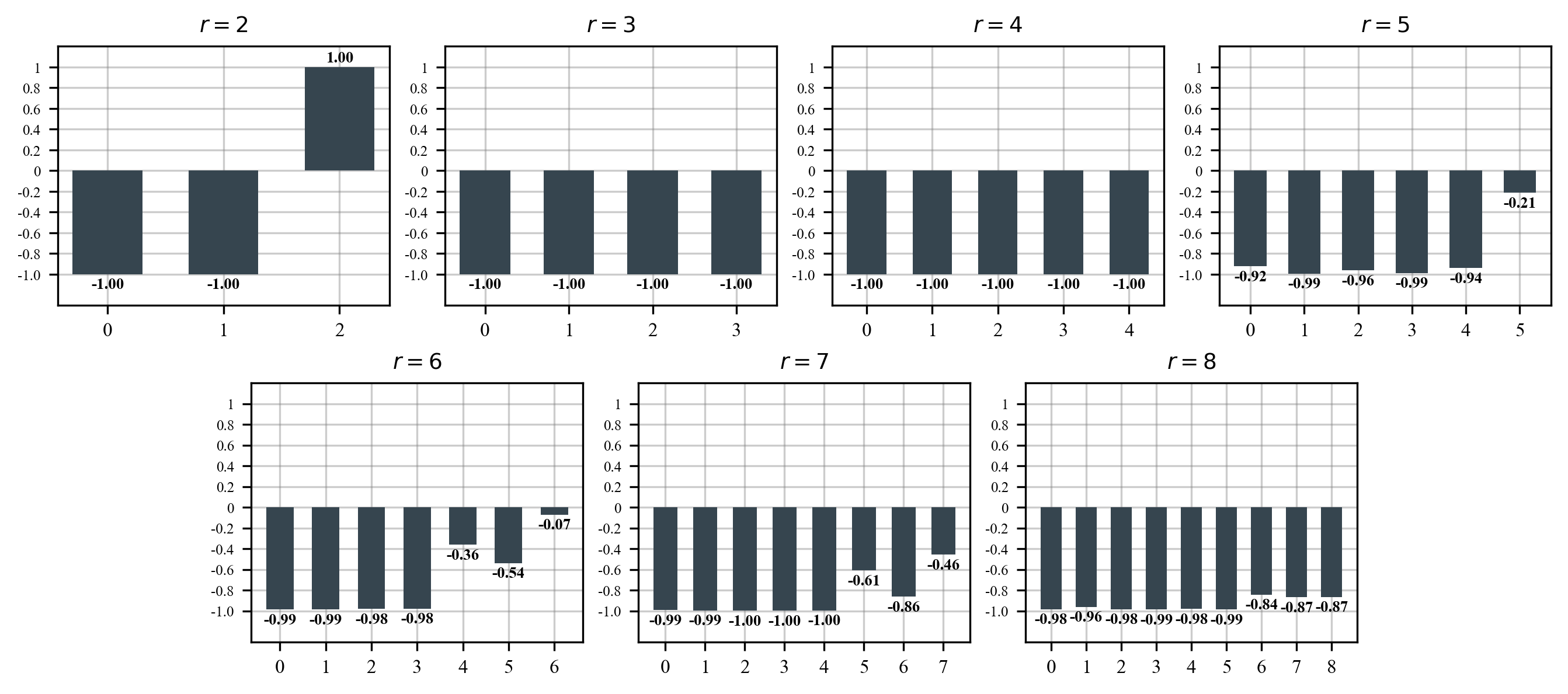}
    \vspace{-10pt}
    \caption{Spearman Correlation between total sheaf effective resistance and accuracy in TreeNeighborsMatch  on each layer across the problem radius $2 \leq r \leq 8$ (left to right).}
    \label{fig:spermanALLr}
    \vspace{-10pt}
\end{figure}
The effective resistance reported in Figure \ref{fig:sprm-tnm} and Figure \ref{fig:spermanALLr} is the mean of total sheaf effective resistance throughout the batches in each epoch, for each layer. The correlation is then calculated over the pairs $(mean(R_{\gF}^{batches}), train~acc)$.

\textbf{Jacobian norm.} Still in the TreeNeighborsMatch, we compute Spearman rank correlation  between the Jacobian norm and the sheaf effective resistance. As predicted by the theory, we obtained a strong negative correlation between them. Moreover, the plots showing correlation with the Jacobian (\cref{fig:jacobian_vs_resistance}) and the plots showing correlation with accuracy (\cref{fig:spermanALLr}) exhibit the same behavior. This result further supports the hypothesis that a reduction in sheaf effective resistance enhances the Jacobian sensitivity, thereby assisting in the mitigation of over-squashing. 
\begin{figure}[h]
    \centering
    \vspace{-7pt}
    \includegraphics[width=0.83\linewidth]{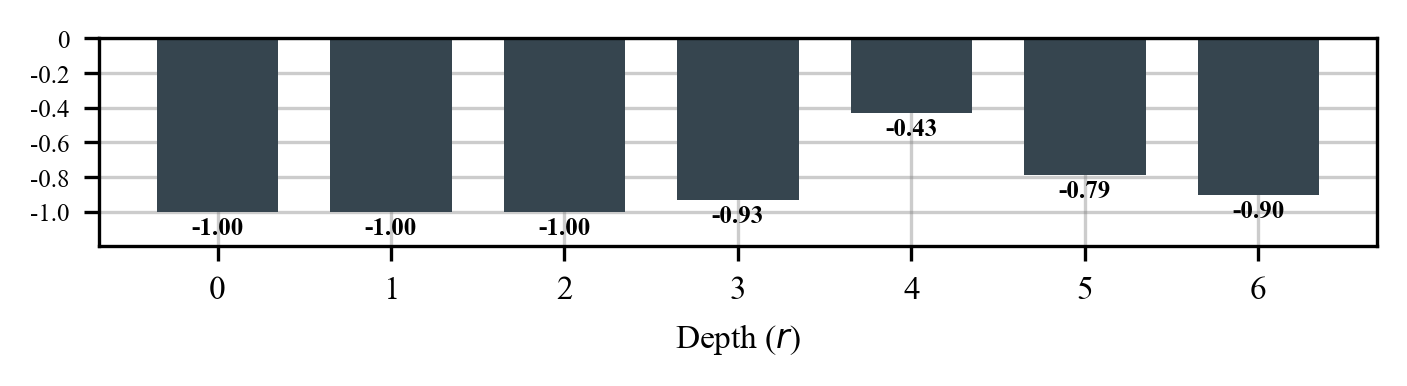}
    \caption{Spearman Correlation between total sheaf effective resistance and Jacobian norm for the TreeNeighborsMatch task on each layer across the problem radius $2 \leq r \leq 8$ (left to right).}
    \label{fig:jacobian_vs_resistance}
    \vspace{-15pt}
\end{figure}

\subsection{Short-range Over-squashing}\label{app:short_range_ovsq}

\cite{mishayev2025short} proposes a short-range over-squashing scenario, instantiated through the two-radius problem. In their experiments, SetTransformer achieves 100\% accuracy across all tested graph sizes, while message-passing GNNs struggle to fit the data, with GAT being the best-performing architecture among GCN, GIN, and SAGE. We evaluate FlatNSD on the same task, fixing $d=2$ and the remaining parameters as in \cite{mishayev2025short}. Notably, FlatNSD with a SetTransformer backbone and hidden dimension 32 achieves 100\% accuracy across all graph sizes, with fast convergence (under 35 epochs).

To better isolate the contribution of the sheaf, we also evaluate FlatNSD with a SAGE backbone. Its performance starts to decay on the two-radius problem with 50 nodes, but remains above that of SAGE alone. This highlights the importance of choosing an appropriate parametrization for the backbone GNN, revealing a similar dichotomy to that observed between the parametrizations used in the TreeNeighborsMatch experiment.

\subsection{Intervention on the TreeNeighborsMatch problem}\label{app:intervention}

To examine the distinction between correlation and mechanism more directly, we perform a controlled intervention on TreeNeighborsMatch with fixed stalk dimension. We compare a trivial flat bundle, with $\gF_v=I_d$ at every node, against an oracle endpoint intervention in which all maps remain fixed to the identity except those at the root and the task-relevant leaf.

The trivial bundle attains the maximum possible total resistance. Under this configuration, the model fails to fit the training set, remaining near $0.045$ training accuracy. By contrast, allowing nontrivial maps only at the two task-relevant endpoints substantially lowers their pairwise sheaf effective resistance, and the model converges almost immediately, after only two iterations. Since the graph topology and all remaining restriction maps are unchanged, this controlled intervention shows that modifying the flat bundle can remove a communication obstruction that persists under the maximum-resistance trivial bundle. This provides more direct evidence for the role of sheaf effective resistance in the mechanism underlying information propagation, beyond the correlation observed during learning.

\section{Contributions and the role of FlatNSD}\label{app:sheaf_models}

The main contributions of this work are theoretical rather than architectural. In particular, we derive a closed-form expression for pairwise sheaf effective resistance for flat bundles (\cref{lemma:sheaf_res}), a closed-form expression for total sheaf effective resistance, together with an efficient way of computing it (\cref{theo:totalsheafresistance_flatbundle}), and establish a connection between sheaf effective resistance and bounds on the message-passing Jacobian (\cref{theo:Jacobian&resistance,theo:Jacobian&resistance}). We further extend \cref{theo:black}, \cref{theo:Jacobian&resistance}, and \cref{cor:totalsheaf} to bipartite graphs by extending the corresponding results for graph effective resistance (\cref{lemma:bipart_pseudo}, \cref{prop:bipart_res}, and \cref{theo:bipart_bound}). We also provide a result characterizing the existence of flat bundles that strictly decrease effective resistance (\cref{prop:existence}) and empirically study how sheaf effective resistance evolves during learning. FlatNSD is introduced as a simplified model used to isolate and verify these mechanisms empirically, rather than as the main contribution of the paper.

Our analysis builds upon several existing results. Sheaf effective resistance was introduced by \cite{hansen2019toward}; our contribution is to recognize and develop its use as a tool for understanding information propagation in sheaf-based GNNs. We further build upon the flat sheaf Laplacian decomposition of \cite{bamberger2024bundle}, the connection between graph effective resistance and message-passing Jacobian bounds developed by \cite{black2023understanding}, and the Neural Sheaf Diffusion (NSD) architecture of \cite{bodnar2022neural}. The results developed here combine these ingredients to connect the geometry of learned sheaves, effective resistance, and over-squashing.

To further clarify the role of FlatNSD, we briefly compare it with the sheaf-based architectures considered throughout this work. FlatNSD intentionally removes architectural components that are not necessary for studying the resistance mechanism, while retaining the basic sheaf diffusion structure:
\begin{itemize}
\item \textbf{FlatNSD:} learns one orthogonal map $\gF_v$ per vertex $v$ and uses the same diffusion process as NSD. The resulting sheaf is a flat $O(d)$-bundle, which allows the decomposition $L_{\gF}=\rmO^\top(L_G\otimes I_d)\rmO$ and makes the theoretical analysis particularly transparent. \looseness=-1

\item \textbf{Orthogonal NSD:} learns one orthogonal restriction map $\gF_{v\unlhd e}$ for each vertex-edge incidence pair $v\unlhd e$. Its diffusion operator is the sheaf Laplacian obtained from the discretized heat equation on the sheaf. FlatNSD can be viewed as a constrained version of this construction in which the incidence maps are induced by a single map per node.

\item \textbf{BuNN:} also learns one orthogonal map $\gF_v$ per vertex, but uses the corresponding heat kernel as its diffusion operator. It therefore shares the same flat-bundle parameterization while differing in how diffusion is performed.

\item \textbf{CSNN:} learns two conformal maps $\rmS_v$ and $\rmT_v$ for each node and constructs the sheaf on a directed graph. Its diffusion operator is formed from a product of in- and out-Laplacians, making it more expressive and structurally different from the flat orthogonal setting considered in our analysis.

\end{itemize}

These distinctions are important for interpreting the experimental results. FlatNSD is not proposed as a replacement for more expressive sheaf-based architectures. Instead, its restricted parameterization provides a controlled setting in which the theoretical predictions can be tested with as few additional mechanisms as possible. In this sense, FlatNSD serves as a basic sheaf model for validating the relationship between learned sheaf geometry, effective resistance, and over-squashing, while architectures such as NSD, BuNN, and CSNN incorporate additional modeling choices that may affect performance beyond this mechanism.

\section{More about sheaf effective resistance}\label{app:sheaf_eff_resistance_more}

\citet{hansen2019toward} introduced \textit{cosheaf effective resistance} on a cell complex\footnote{A graph is a $1$-dimensional cell complex. Cellular sheaves can be defined in this more general context, but considering them into GNNs realm brings complications we do not discuss in this work.}. Here, we adapt their definition to a graph $G$  with $n$ vertices and $\gF$ a sheaf over $G$ whose stalks have dimension $d.$ 
\begin{definition}\label{df:sheaf_res_min}
The sheaf effective resistance between nodes $u$ and $v$ is given by the solution to the optimization problem
    \begin{equation}
        \min\limits_{c \in C^1(G, \gF)} \|c\|^2 \,\,\,\,\mbox{s.t} \,\,\,\, \delta^\top c = \mathbbm{1}_u-\mathbbm{1}_v,
    \end{equation}
    where $\mathbbm{1}_v$ is a binary vector of length $nd$ where only the $d$ entries corresponding to node $v$ is $1$.
\end{definition}
If the sheaf is trivial, then the effective resistance is the solution to $\min\{\|c\|^2 \mid  Bc = 1_u-1_v\}$, where $B$ is the incidence matrix of $G$. This is equivalent to $ \operatorname{R}(u,v) = (1_u - 1_v)^{\top}L^+_G(1_u - 1_v).$

 We check the equivalence between \cref{df:sheaf_res_min} and \cref{def:sheaf_res} in the following proposition:
 
\begin{proposition}
\cref{df:sheaf_res_min} and \cref{def:sheaf_res} are equivalent.
\end{proposition}
\begin{proof}
    Recall that for any matrix $A^\top A$, it holds that  $x^\top(A^\top A)^+x = \operatorname{min}\{\|y \|^2 \mid A^\top y = x \}$. Moreover, $L_{\gF} = \delta^\top\delta$. Consequently: 
    \begin{align*}
        \operatorname{R}_{\gF}(u,v) &= (\mathbbm{1}_u - \mathbbm{1}_v)^{\top} L_{\gF}^+ (\mathbbm{1}_u - \mathbbm{1}_v) \\
        &=\operatorname{min}\{\|c\|^2 \mid \delta^\top c = (\mathbbm{1}_u - \mathbbm{1}_v)\}.   
    \end{align*}
\end{proof}

In this work we also consider the total sheaf effective resistance, which $\operatorname{R}_{\gF}$ is the sum of the sheaf effective resistance $\operatorname{R}_{\gF}(u,v)$ for all pairs of nodes. With vertex-major indexing, write \(L_{\mathcal F}^{+}\) as \(n^2\) blocks \((L_{\mathcal F}^{+})^{ij}\), each of dimension \(d\times d\) in the following way:\looseness-1
  $$L_{\gF}^+ = \begin{pmatrix}
        (L_{\gF}^+)^{11} & \dots & (L_{\gF}^+)^{1n} \\
        \vdots & \ddots & \vdots \\
        (L_{\gF}^+)^{n1} & \dots & (L_{\gF}^+)^{nn}
   \end{pmatrix},$$
then a straightforward calculation leads to:
\begin{proposition}\label{prop:sheaftotalresistance_general}
Let $\langle \cdot, \cdot \rangle_F$ denote the Frobenius inner product, and $\textbf{1}$ denote a matrix whose entries are all equal to one. Let $\gF$ be a sheaf over a connected graph $G$. Then
$$\operatorname{R}_{\gF} = (n-1)\sum\limits_{i \in V} \langle(L^+_{\gF})^{ii},\textbf{1}\rangle_F - 2\sum\limits_{i < j} \langle(L^+_{\gF})^{ij},\textbf{1}\rangle_F,$$
where $\textbf{1} \in \sR^{n \times n}$ is a matrix of ones.
\end{proposition}
\begin{proof}
  Here we use a numerical index notation $i,j$, which we believe will help to follow the argument. 
  
 Observe that $R_{\gF}(i,j) = \sum\limits_{k,l = 1}^{d} \left((L_{\gF}^+)^{ii}_{kl} + (L_{\gF}^+)^{jj}_{kl}\right) - 2 \sum\limits_{k,l = 1}^{d}(L_{\gF}^+)^{ij}_{kl}$, for each pair of nodes $i,j \in V$.

   To calculate the sheaf effective resistance we need sum all $R_{\gF}(i,j)$. So
\begin{align*}
    R_{\gF} &= (n-1)\left(\sum\limits_{ij=1}^d (L_{\gF}^+)^{11}_{ij} + \dots + \sum\limits_{ij=1}^d (L_{\gF}^+)^{nn}_{ij} \right) - 2\left(\sum\limits_{ij=1}^d (L_{\gF}^+)^{12}_{ij}+\dots+(L_{\gF}^+)^{1n}_{ij}\right) \\&- 2\left((L_{\gF}^+)^{23}_{ij}+ (L_{\gF}^+)^{2n}_{ij}\right) - \dots - 2\sum\limits_{ij=1}^d (L_{\gF}^+)^{n-1n}_{ij} \\ &=(n-1)\sum\limits_{i \in V} \langle(L^+_{\gF})^{ii},\textbf{1}\rangle_F - 2\sum\limits_{i < j} \langle(L^+_{\gF})^{ij},\textbf{1}\rangle_F.
\end{align*}
\end{proof}

\section{Proofs and verifications}\label{app:proofs}
\textbf{\cref{Ex:sheaves_reduce_resistance}.} Consider the graph with only two vertices $u$ and $v$, and an edge e between them.
\begin{enumerate}
    \item If $\gF$ is trivial, the sheaf Laplacian is $\begin{pmatrix}
        1 &-1 \\
        -1 & 1
    \end{pmatrix}$ whose pseudo-inverse is $\begin{pmatrix}
         1/4 &-1/4 \\
        -1/4 & 1/4
    \end{pmatrix}$, $\mathbbm{1}_u = (1,0)$, and $\mathbbm{1}_v = (0,1)$. Thus, $R(u,v) = 1$.


        \item If $\gF$ is a sheaf such that all stalks are isomorphic to $\sR^2$, $\rmapf{u}{e} = \begin{pmatrix}
            \cos(x) & -\sin(x) \\ \sin(x) & \cos(x)
        \end{pmatrix}$, and 
        $\rmapf{v}{e} = \begin{pmatrix}
            \cos(y) & -\sin(y) \\ \sin(y) & \cos(y)
        \end{pmatrix}$, the sheaf Laplacian is $$\begin{pmatrix}
            1&0 & -\cos(x-y) & -\sin(x-y)\\
            0 & 1 & \sin(x-y) & -\cos(x-y) \\
            -\cos(x-y)& \sin(x-y) & 1 & 0 \\
            -\sin(x-y) & -\cos(x-y) & 0 & 1
        \end{pmatrix},$$ whose pseudo-inverse is 
        
        $$\begin{pmatrix}
            1/4 &0 & -\cos(x-y)/4 & -\sin(x-y)/4\\
            0& 1/4& \sin(x-y)/4 & -\cos(x-y)/4\\
            -\cos(x-y)/4 & \sin(x-y)/4 & 1/4 & 0 \\
            -\sin(x-y)/4 & -\cos(x-y)/4 &0 &1/4 
        \end{pmatrix},$$ with  $\mathbbm{1}_u = (1,1,0,0)$, and $\mathbbm{1}_v = (0,0,1,1)$. Thus, $R_{\gF}(u,v) = \cos(x-y) + 1$.

\end{enumerate}

\textbf{\cref{prop:sheaf_lap_pinv}.}
    The pseudo-inverse of $L_{\gF}$ is $L_{\gF}^+ = \rmO^{\top}(L_G^+ \otimes I_d)\rmO$.
\begin{proof}
    Using the fact that $(A \otimes B)^+ = A^+ \otimes B^+$, and that for orthogonal matrices $U$ and $V$ we have $(UAV)^+ = V^\top A^+ U^\top$, since the matrix $\rmO$ is orthogonal:
    \begin{align*}
        L_{\gF}^+ &= \left(\rmO^{\top}(L_G \otimes I_d)\rmO\right)^+ \\ &= (\rmO)^+(L_G \otimes I_d)^+(\rmO^{\top})^+ \\ &= (\rmO)^{-1}(L_G^+ \otimes I_d^+)(\rmO^{\top})^{-1} \\ &= \rmO^{\top}(L_G^+ \otimes I_d)\rmO.
    \end{align*}
\end{proof}

\textbf{\cref{lemma:sheaf_res}.}
Let $\mathbf{1}_d \in \sR^d$ be a vector of ones, and $M_{uv} = \mathbf{1}_d^{\top}\gF_u^\top\gF_v\mathbf{1}_d$. Then
  $$\operatorname{R}_{\mathcal{F}}(u,v) = d((L_G^+)_{uu}+(L_G^+)_{vv}) - 2M_{uv}(L_G^+)_{uv}.$$

\begin{proof}
    Using \cref{def:sheaf_res} and \cref{prop:sheaf_lap_pinv} we obtain
    \begin{align*}
        R_{\gF}(u,v) = (\mathbbm{1}_u - \mathbbm{1}_v)^{\top}L_{\mathcal{F}}^+(\mathbbm{1}_u - \mathbbm{1}_v) = (\mathbbm{1}_u - \mathbbm{1}_v)^{\top}\rmO^{\top}(L_{G}^+ \otimes I_d)\rmO(\mathbbm{1}_u - \mathbbm{1}_v).
    \end{align*}

    We first note that $\rmO = \sum_{i=1}^n (\rve_i\rve_i^{\top}) \otimes \gF_i$ and $\mathbbm{1}_u = \rve_u \otimes \mathbf{1}_d$, where $\{\rve_i\}_{i=1}^n$ is the standard basis of $\sR^n$. Thus
    \begin{align*}
        \rmO\mathbbm{1}_u = \left(\sum_{i=1}^n (\rve_i\rve_i^{\top}) \otimes \gF_i\right)(\rve_u \otimes \mathbf{1}_d) = \sum_{i=1}^n (\rve_i\rve_i^{\top}\rve_u) \otimes (\gF_i\mathbf{1}_d) = \rve_u \otimes (\gF_u\mathbf{1}_d).
    \end{align*}

    Let $s_u = \gF_u\mathbf{1}_d$, then $\|s_u\|_2^2 = \mathbf{1}_d^{\top}\gF_u^{\top}\gF_u\mathbf{1}_d = d$. Also, define $M_{uv} = s_u^{\top}s_v$. Therefore
    \begin{align*}
        R_{\gF}(u,v) &= ((\rve_u \otimes s_u) - (\rve_v \otimes s_v))^{\top}(L_G^+ \otimes I_d)((\rve_u \otimes s_u) - (\rve_v \otimes s_v)) \\ &= ((\rve_u^{\top} \otimes s_u^{\top}) - (\rve_v^{\top} \otimes s_v^{\top}))((L_G^+\rve_u \otimes s_u) - (L_G^+\rve_v \otimes s_v)) \\ &= (\rve_u^{\top}L_G^+\rve_u \otimes s_u^{\top}s_u) - (\rve_u^{\top}L_G^+\rve_v \otimes s_u^{\top}s_v) - (\rve_v^{\top}L_G^+\rve_u \otimes s_v^{\top}s_u) + (\rve_v^{\top}L_G^+\rve_v \otimes s_v^{\top}s_v) \\ &= (L_G^+)_{uu}\|s_u\|_2^2 - (L_G^+)_{uv}s_u^{\top}s_v - (L_G^+)_{vu}s_v^{\top}s_u + (L_G^+)_{vv}\|s_v\|_2^2 \\ &= d((L_G^+)_{uu} + (L_G^+)_{vv}) - 2M_{uv}(L_G^+)_{uv}.
    \end{align*}
\end{proof}

\begin{lemma}\label{lemma:resistance_with_adj}{(Lemma 3.5 \cite{black2023understanding})}
    Let $G$ be a non-bipartite graph. Let $u$ and $v$ be two vertices in $G$. Then 
    $$\operatorname{R}(u,v) = \sum\limits^{\infty}_{j=0}\left( \frac{1}{d_u}(\hat{A}^j)_{uu} + \frac{1}{d_v}(\hat{A}^j)_{vv} - \frac{2}{\sqrt{d_ud_v}}(\hat{A}^j)_{uv}\right).$$
\end{lemma}

\begin{lemma}\label{lemma:jacobian_and_graphadjacency}{(Lemma 3.2 \cite{black2023understanding})} Let $u, v \in V$. Suppose 
$\|\nabla \phi_t\| \leq \alpha$ and $\max\{\|\nabla \psi_t\|, 1\} \leq \beta$, for all $t = 0,...,r$, where $\nabla f$ denotes the Jacobian of a map $f$.  Then
    $$    \left\|\frac{\partial h_u^{(r)}}{\partial x_v}\right\| \leq (2\alpha \beta)^{r} \sum\limits^{r}_{j=0}(\hat{A}^j)_{uv}.$$
\end{lemma}

Then we can prove:

\textbf{\cref{theo:Jacobian&resistance}.} Let $G$ be a non-bipartite graph, and let $u, v \in V$. Suppose 
$\|\nabla \phi_l\| \leq \alpha$ and $\max\{\|\nabla \psi_t\|, 1\}| \leq \beta$. 
Define $d_{\min} = \min\{d_u, d_v\}$, $d_{\max} = \max\{d_u, d_v\}$, and fix $d$ the dimension of the stalks. 
Let $\max\{|\mu_2|,..., |\mu_n|\} \le \mu$. If $(L_G^+)_{uv} \geq 0$, then $\left\|\frac{\partial h_u^{(r)}}{\partial x_v}\right\|$ is upper bounded by 
\begin{equation*}
    (2\alpha\beta)^{r} \frac{d_{max}}{2}\Biggl(\frac{2}{d_{min}}\Biggr(r + 1 +\frac{\mu^{r+1}}{1-\mu}\Biggr)+4(L_G^+)_{uv} - \frac{\operatorname{R}_{\gF}(u,v)}{d} \Biggr).
\end{equation*}
If $(L_G^+)_{uv} < 0$, it is upper bounded by
\begin{equation*}
    (2\alpha\beta)^{r} \frac{d_{max}}{2}\Biggl(\frac{2}{d_{min}}\Biggr(r + 1 +\frac{\mu^{r+1}}{1-\mu}\Biggr) - \frac{\operatorname{R}_{\gF}(u,v)}{d} \Biggr).
\end{equation*}
\begin{proof}
By \cref{lemma:resistance_with_adj}  and \cref{eq:resistances},

$$\operatorname{R}_{\mathcal{F}}(u,v) = d\sum\limits^{\infty}_{j=0}\left( \frac{1}{d_u}(\hat{A}^j)_{uu} + \frac{1}{d_v}(\hat{A}^j)_{vv} - \frac{2}{\sqrt{d_ud_v}}(\hat{A}^j)_{uv}\right) + 2(d - M_{uv})(L_G^+)_{uv}.$$
So
\begin{align*}
    \sum\limits^{r}_{j=0}(\hat{A}^j)_{uv} = \frac{\sqrt{d_ud_v}}{2}&\Bigg(\sum\limits^{\infty}_{j=0}\left(\frac{1}{d_u}(\hat{A}^j)_{uu} + \frac{1}{d_v}(\hat{A}^j)_{vv}\right) - \\ &\sum\limits^{\infty}_{j=r+1}\frac{2}{\sqrt{d_ud_v}}(\hat{A}^j)_{uv}  + \frac{2}{d}(d - M_{uv})(L_G^+)_{uv} -\frac{1}{d}\operatorname{R}_{\gF}(u,v)\Bigg)
\end{align*}

By \cref{lemma:jacobian_and_graphadjacency},
\begin{align*}
    \left\|\frac{\partial h_u^{(r)}}{\partial x_v}\right\| &\leq (2\alpha \beta)^{r} \sum\limits^{r}_{j=0}(\hat{A}^j)_{uv} \\ & \leq (2\alpha \beta)^{r}\frac{d_{max}}{2}\Bigg(\sum\limits^{\infty}_{j=0}\left(\frac{1}{d_u}(\hat{A}^j)_{uu} + \frac{1}{d_v}(\hat{A}^j)_{vv}\right) -\\& \sum\limits^{\infty}_{j=r+1}\frac{2}{\sqrt{d_ud_v}}(\hat{A}^j)_{uv} + \frac{2}{d}(d - M_{uv})(L_G^+)_{uv}-\frac{1}{d}\operatorname{R}_{\gF}(u,v)\Bigg)
\end{align*} 

The proof of Theorem 3.3 in \citet{black2023understanding} shows that $\sum\limits^{\infty}_{j=0}\left(\frac{1}{d_u}(\hat{A}^j)_{uu} + \frac{1}{d_v}(\hat{A}^j)_{vv}\right) - \sum\limits^{\infty}_{j=r+1}\frac{2}{\sqrt{d_ud_v}}(\hat{A}^j)_{uv}$ is bounded by $\dfrac{2}{d_{min}}\left((r +   1) +\frac{\mu^{r+1}}{1-\mu}\right)$. Since $-d \leq M_{uv} \leq d$, if $(L_G^+)_{uv} \geq 0$ we obtain 
  $$\left\|\frac{\partial h_u^{(r)}}{\partial x_v}\right\| \leq (2\alpha\beta)^{r} \frac{d_{max}}{2}\left(\dfrac{2}{d_{min}}\left(r +   1 +\frac{\mu^{r+1}}{1-\mu}\right) +4(L_G^+)_{uv} - \frac{\operatorname{R}_{\gF}(u,v)}{d} \right).$$

  If $(L_G^+)_{uv} < 0$, we obtain 

  $$\left\|\frac{\partial h_u^{(r)}}{\partial x_v}\right\| \leq (2\alpha\beta)^{r} \frac{d_{max}}{2}\left(\dfrac{2}{d_{min}}\left(r +   1 +\frac{\mu^{r+1}}{1-\mu}\right) - \frac{\operatorname{R}_{\gF}(u,v)}{d} \right).$$
\end{proof}

\textbf{\cref{theo:totalsheafresistance_flatbundle}.}
    Let $\langle \cdot, \cdot \rangle_F$ denote the Frobenius inner product and define $M$ as in \cref{lemma:sheaf_res}. Then
    $$\operatorname{R}_{\mathcal{F}} = dn\operatorname{Tr}(L_G^+) - \langle M, L_G^+\rangle_F = d\operatorname{R} - \langle M, L_G^+\rangle_F$$
\begin{proof}
On one hand, by \cref{lemma:sheaf_res}, we have
    \begin{align*}
     \operatorname{R}_{\mathcal{F}} &= d\sum_{u < v} \left((L_G^+)_{uu}+(L_G^+)_{vv}\right) - 2 \sum_{u < v } M_{uv}(L_G^+)_{uv}  \\ &= d(n-1) \sum_{u \in V} (L_G^+)_{uu} - 2 \sum_{u < v } M_{uv}(L_G^+)_{uv} \\ &= d(n-1) \operatorname{Tr}(L_G^+) - 2 \sum_{u < v } M_{uv}(L_G^+)_{uv}
    \end{align*}

On the other hand, observe that 
    \begin{align*}
        \langle M, L_G^+\rangle_F = \sum\limits_{u,v \in V}M_{uv}(L^+_G)_{uv} &= \sum\limits_{u \in V}M_{uu}(L^+_G)_{uu} + \sum\limits_{u \neq v}M_{uv}(L^+_G)_{uv}\\ &=d \operatorname{Tr}(L_G^+) + 2\sum\limits_{u < v}M_{uv}(L^+_G)_{uv},
    \end{align*}
    
where $u < v$ means we are considering only the upper triangular part of the matrix, and the equality holds because both $L^+_G$ and $M$ are symmetric. Thus
\begin{equation*}
    2\sum\limits_{u < v}M_{uv}(L^+_G)_{uv} = \langle M, L_G^+\rangle_F - d\operatorname{Tr}(L_G^+),
\end{equation*}

Combining the above, we obtain 
$$\operatorname{R}_{\mathcal{F}} = d(n-1)\operatorname{Tr}(L_G^+) - (\langle M, L_G^+\rangle_F - d \operatorname{Tr}(L_G^+)) = dn\operatorname{Tr}(L_G^+) - \langle M, L_G^+\rangle_F.$$
\end{proof}

\textbf{\cref{prop:existence}.} Let $G$ be a connected graph with at least two nodes and $\gF$ a flat $O(d)$-bundle over $G$. Then $\frac{R_{\gF}}{d} \leq R$,
with equality holding if and only if $\gF_v\mathbf{1}_d=c$ for some $c\in\sR^d$ and all $v\in V$. In particular, for every $d\geq 1$, there exists a flat $O(d)$-bundle $\gF$ such that $\frac{R_{\gF}}{d}<R$.

\begin{proof}
    Recall that $S$ is a matrix with $S_u = \mathbf{1}_d^{\top}\gF_u^{\top}$ and $M = SS^{\top}$, therefore $$\langle M, L_G^+ \rangle = \Tr(S^{\top} L_G^+ S) \geq 0,$$ since $L_G^+ \succeq 0$, and then $\frac{R_{\gF}}{d} \leq R$ by \cref{theo:totalsheafresistance_flatbundle}. If the underlying graph $G$ is connected, $\ker L_G = (\mathbf{1}_n)$, thus $\frac{R_{\gF}}{d} = R$ precisely when $\gF_v\mathbf{1}_d = c$, for some $c \in \sR^d$ and all $v \in V$. Therefore the inequality is strict whenever $\gF_v \mathbf{1}_d$ is non-constant. We proceed to show a family of sheaves for which effective resistance is decreased.

    Consider any non-trivial vertex partition $V = A \sqcup B$ and define $\varepsilon_v = \begin{cases} +1, & v \in A, \\ -1, & v \in B \end{cases}$. Set $\gF_v = \operatorname{diag}(\varepsilon_v, 1, \dots, 1) \in O(d)$, then for two nodes $v \in A$ and $u \in B$ we get $S = \begin{bmatrix} \varepsilon & \mathbf{1}_n & \dots & \mathbf{1}_n \end{bmatrix}$, therefore $\operatorname{Tr}(S^{\top} L_G^+ S) = \varepsilon^{\top} L_G^+ \varepsilon + \sum_{j=2}^d \mathbf{1}_n^{\top} L_G^+ \mathbf{1}_n = \varepsilon^{\top} L_G^+ \varepsilon > 0$, since $\varepsilon$ is non-constant. By \cref{theo:totalsheafresistance_flatbundle}, we conclude that $R_{\gF} = dR - \operatorname{Tr}(S^{\top}L_G^+S) \implies \frac{R_{\gF}}{d} < R$.
\end{proof}

\textbf{\cref{cor:totalsheaf}.}     Let $G$ be a non-bipartite graph, and let $u, v \in V$. Suppose 
$\|\nabla \phi_l\| \leq \alpha$ and $\max\{\|\nabla \psi_t\|, 1\}| \leq \beta$. 
Define $d_{\min} = \min\{d_u, d_v\}$, $d_{\max} = \max\{d_u, d_v\}$, and let $d$ be the dimension of the stalks. 
Let $\max\{|\mu_2|,..., |\mu_n|\} \le \mu$. Then
  $$\sum\limits_{u\neq v \in V} \left\|\frac{\partial h_u^{(r)}}{\partial x_v}\right\| \leq (2\alpha\beta)^{r} \frac{d_{max}}{2}\left(\dfrac{n(n-1)}{d_{min}}\left(r +   1 +\frac{\mu^{r+1}}{1-\mu}\right) - \frac{\operatorname{R}_{\gF}}{d} \right).$$
\begin{proof}
    Observe that 
    \begin{align*}
        &\sum\limits_{u\neq v \in V}(2\alpha\beta)^{r+1} \frac{d_{max}}{2}\left(\dfrac{2}{d_{min}}\left(r +   1 +\frac{\mu^{r+1}}{1-\mu}\right) +4(L_G^+)_{uv} - \frac{\operatorname{R}_{\gF}(u,v)}{d} \right) \\ &= (2\alpha\beta)^{r} \frac{d_{max}}{2}\left(\dfrac{n(n-1)}{d_{min}}\left(r +   1 +\frac{\mu^{r+1}}{1-\mu}\right) - 4\operatorname{Tr}(L^+_G) - \frac{2\operatorname{R}_{\gF}}{d} \right)\\ &\leq (2\alpha\beta)^{r} \frac{d_{max}}{2}\left(\dfrac{n(n-1)}{d_{min}}\left(r +   1 +\frac{\mu^{r+1}}{1-\mu}\right) - \frac{\operatorname{R}_{\gF}}{d} \right),
    \end{align*}

    because  the sum of each line in $(L^+_G) $ equals to $0$, $\operatorname{Tr}(L^+_G) \geq 0$, and $\operatorname{R}_{\gF} \leq 2\operatorname{R}_{\gF}$.

Similarly, 
    \begin{align*}
        &\sum\limits_{u\neq v \in V}(2\alpha\beta)^{r+1} \frac{d_{max}}{2}\left(\dfrac{2}{d_{min}}\left(r +   1 +\frac{\mu^{r+1}}{1-\mu}\right) - \frac{\operatorname{R}_{\gF}(u,v)}{d} \right) \\ &= (2\alpha\beta)^{r} \frac{d_{max}}{2}\left(\dfrac{n(n-1)}{d_{min}}\left(r +   1 +\frac{\mu^{r+1}}{1-\mu}\right)  - \frac{2\operatorname{R}_{\gF}}{d} \right)\\ &\leq (2\alpha\beta)^{r} \frac{d_{max}}{2}\left(\dfrac{n(n-1)}{d_{min}}\left(r +   1 +\frac{\mu^{r+1}}{1-\mu}\right) - \frac{\operatorname{R}_{\gF}}{d} \right).
 \end{align*}
\end{proof}

\textbf{\cref{lemma:bipart_pseudo}.} Let $G$ be a connected bipartite graph. Then $\Delta_G^+ = \frac{1}{2} P_{-} + \sum_{k=0}^{\infty} (\hat{A}^k - P_+ - (-1)^kP_-)$, where $P_+$ and $P_-$ are the outer product of the eigenvectors with eigenvalues $1$ and $-1$, respectively.

\begin{proof}
    Let $G = (U \sqcup W, E)$, $\varepsilon_v = \begin{cases} +1, & u \in W \\ -1, & v \in U \end{cases}$, and $\operatorname{vol}(G) = \sum_{v \in V} d_v$. The kernel of the Laplacian is $\ker L_G = (\mathbf{1}_n)$, and the symmetrically normalized Laplacian is $\Delta_G = I - D^{-1/2}AD^{-1/2} = I - \hat{A}$. Thus
    $$x \in \ker \Delta_G \iff (I - \hat{A})x = 0 \iff \hat{A}x = x.$$
    We also have 
    $$x \in \ker \Delta_G \iff D^{-1/2}L_GD^{-1/2}x = 0 \iff L_G(D^{-1/2}x) = 0 \iff D^{-1/2}x \in \ker L_G.$$
    We conclude that if $x \in \ker \Delta_G$, then $x = cD^{1/2}\mathbf{1}_n$, and therefore the normalized eigenvector of $\hat{A}$ corresponding to the eigenvalue 1 is $q_+ = \frac{D^{1/2}\mathbf{1}_n}{\sqrt{\operatorname{vol}(G)}}$. Now let $(u,v) \in E$, then $\varepsilon_u = -\varepsilon_v$, thus 
    $$(A\varepsilon)_v = \sum_{u \in \gN(v)} \varepsilon_u = -d_v\varepsilon_v,$$
    i.e. $A\varepsilon = -D\varepsilon$. This yields $$\hat{A}D^{1/2}\varepsilon = D^{-1/2}A\varepsilon = D^{-1/2}(-D\varepsilon) = -D^{1/2}\varepsilon,$$
    and we conclude that the normalized eigenvector with eigenvalue $-1$ is $q_- = \frac{D^{1/2}\varepsilon}{\sqrt{\operatorname{vol}(G)}}$. Defining $P_+ = q_+q_+^{\top}$, $P_- = q_-q_-^{\top}$ and $P_i = q_iq_i^{\top}$, where the $q$'s are the eigenvectors of $\hat{A}$, we get 
    $$\hat{A} = Q \operatorname{diag}(1, -1, \mu_2, \dots, \mu_{n-1})\:Q^{\top} = P_+ - P_- + \sum_{i=2}^{n-1} \mu_iP_i.$$
    Furthermore, since $B = \{q_+, q_-, q_2, \dots, q_{n-1}\}$ forms an orthonormal basis, we get $I = P_+ + P_- + \sum_{i=2}^{n-1} P_i$, therefore
    $$\Delta_G = I - \hat{A} = 2P_- + \sum_{i=2}^{n-1} (1 - \mu_i) P_i \implies \Delta_G^+ = \frac{1}{2}P_- + \sum_{i=2}^{n-1} \frac{1}{1 - \mu_i} P_i.$$
    Since $|\mu_i| < 1$, we have $\frac{1}{1 - \mu_i} = \sum_{k=0}^\infty \mu_i^k$, thus 
    $$\Delta_G^+ = \frac{1}{2}P_- + \sum_{i=2}^{n-1}\sum_{k=0}^\infty \mu_i^k P_i = \frac{1}{2}P_- + \sum_{k=0}^\infty\sum_{i=2}^{n-1} \mu_i^k P_i.$$ But $\hat{A}^k = P_+ + (-1)^kP_- + \sum_{i=2}^{n-1} \mu_i^kP_i \implies \sum_{i=2}^{n-1} \mu_i^kP_i = \hat{A}^k - P_+ - (-1)^kP_-$, which yields
    $$\Delta_G^+ = \frac{1}{2}P_- + \sum_{k=0}^{\infty} (\hat{A}^k - P_+ - (-1)^kP_-).$$
\end{proof}

\textbf{\cref{prop:bipart_res}.} Let $G$ be a connected bipartite graph, $b_{uv} = \mathbbm{1}_u - \mathbbm{1}_v$, $z_{uv} = D^{-1/2}b_{uv}$, and $\gamma_{uv} = z_{uv}^{\top}P_-z_{uv}$. Then $$R(u,v) = \frac{\gamma_{uv}}{2} + \sum_{k=0}^\infty \left( \frac{(\hat{A}^k)_{uu}}{d_u} + \frac{(\hat{A}^k)_{vv}}{d_v} - \frac{2(\hat{A}^k)_{uv}}{\sqrt{d_ud_v}} - (-1)^k\gamma_{uv} \right)$$

\begin{proof}
    The effective resistance can be defined as $z_{uv}^{\top}\Delta_G^+z_{uv}$, where $z_{uv} = D^{-1/2}(\mathbbm{1}_u - \mathbbm{1}_v)$. Therefore, by the previous lemma
    $$R(u,v) = z_{uv}^{\top}\left(\frac{1}{2}P_- + \sum_{k=0}^{\infty} (\hat{A}^k - P_+ - (-1)^kP_-)\right)z_{uv}.$$
    Notice that $q_+^{\top}z_{uv} = \frac{\mathbf{1}_n^{\top}D^{1/2}}{\sqrt{\operatorname{vol}(G)}} D^{-1/2}(\mathbbm{1}_u - \mathbbm{1}_v) = \frac{1}{\sqrt{\operatorname{vol}(G)}}\mathbf{1}_n^{\top}(\mathbbm{1}_u - \mathbbm{1}_v) = 0$, thus $z_{uv}^{\top}P_+z_{uv} = 0$. Define $\gamma_{uv} = z_{uv}^{\top}P_-z_{uv}$, then
    $$R(u,v) = \frac{\gamma_{uv}}{2} + \sum_{k=0}^{\infty} (z_{uv}^{\top}\hat{A}^kz_{uv} - (-1)^k\gamma_{uv}).$$
    Now, notice that
    $$z_{uv}^{\top}\hat{A}^kz_{uv} = \left(\frac{\mathbbm{1}_u^{\top}}{\sqrt{d_u}} - \frac{\mathbbm{1}_v^{\top}}{\sqrt{d_v}}\right)\hat{A}^k\left(\frac{\mathbbm{1}_u}{\sqrt{d_u}} - \frac{\mathbbm{1}_v}{\sqrt{d_v}}\right) = \frac{(\hat{A}^k)_{uu}}{d_u} + \frac{(\hat{A}^k)_{vv}}{d_v} - 2\frac{(\hat{A}^k)_{uv}}{\sqrt{d_ud_v}},
    $$
    thus
    $$R(u,v) = \frac{\gamma_{uv}}{2} + \sum_{k=0}^{\infty} \left(\frac{(\hat{A}^k)_{uu}}{d_u} + \frac{(\hat{A}^k)_{vv}}{d_v} - \frac{2(\hat{A}^k)_{uv}}{\sqrt{d_ud_v}} - (-1)^k\gamma_{uv}\right).$$
\end{proof}

\textbf{\cref{theo:bipart_bound}.} Let $G$ be a connected bipartite graph. Let $u, v \in V = U \sqcup W$, $\|\nabla \phi_l\| \leq \alpha$ and $\max\{\|\nabla \psi_l\|, 1\| \leq \beta$, $d_{\min} = \min\{d_u,d_v\}$, $d_{\max} = \max\{d_u,d_v\}$, and $\overline{\mu} = \max_{i=2,\dots,n-1} |\mu_i| < 1$. Then
    $$\left\|\frac{\partial h_u^{(r)}}{\partial x_v}\right\| \leq (2\alpha\beta)^r \frac{d_{\max}}{2}\left(\frac{2}{d_{\min}}\left(r + 1 + \frac{\overline{\mu}^{r+1}}{1 - \overline{\mu}}\right) + \frac{1}{2}\gamma_{uv} - R(u,v)\right).$$

\begin{proof}
    We have
    $$R(u,v) = \frac{\gamma_{uv}}{2} + \sum_{k=0}^{r} \left(z_{uv}^{\top}\hat{A}^kz_{uv} - (-1)^k\gamma_{uv}\right) + T_r(u,v),$$
    where $T_r(u,v) = \sum_{k=r+1}^{\infty} \left(z_{uv}^{\top}\hat{A}^kz_{uv} - (-1)^k\gamma_{uv}\right)$. We can write $z_{uv}$ in the  orthonormal basis as
    $$z_{uv} = \left(q_-^{\top} z_{uv}\right)q_- + \sum_{i=2}^{n-1} \left(q_i^{\top}z_{uv}\right)q_i,$$
    since $z_{uv} \perp q_+$. Recalling that $\hat{A}^k q_- = (-1)^k q_-$, $\hat{A}^kq_i = \mu_i^k q_i$, $q_- \perp q_i$ for every node, and $q_i \perp q_j$ for $i \neq j$, we get:
    \begin{align*}
        z_{uv}^{\top}\hat{A}^kz_{uv} 
        &= \left[\left(q_-^{\top} z_{uv}\right)q_-^{\top} + \sum_{i=2}^{n-1} \left(q_i^{\top}z_{uv}\right)q_i^{\top}\right]\hat{A}^k \left[\left(q_-^{\top} z_{uv}\right)q_- + \sum_{i=2}^{n-1} \left(q_i^{\top}z_{uv}\right)q_i\right] \\
        &= \left(q_-^{\top} z_{uv}\right)^2 (-1)^k + \sum_{i=2}^{n-1} (q_i^{\top} z)^2 \mu_i^k,
    \end{align*}
    thus $z_{uv}^{\top}\hat{A}^kz_{uv} - \gamma_{uv}(-1)^k = \sum_{i=2}^{n-1} (q_i^{\top} z)^2 \mu_i^k$. Substituting into $T_r(u,v)$:
    $$T_r(u,v) = \sum_{k=r+1}^{\infty} \sum_{i=2}^{n-1} (q_i^{\top} z)^2 \mu_i^k = \sum_{i=2}^{n-1} (q_i^{\top} z)^2 \sum_{k=r+1}^{\infty} \mu_i^k = \sum_{i=2}^{n-1} (q_i^{\top} z)^2 \frac{\mu_i^{r+1}}{1-\mu_i},$$
    since $|\mu_i| < 1$ for all $i = 2, \dots, n-1$ implies $\frac{\mu_i^{r+1}}{1-\mu_i} = \sum_{k=r+1}^{\infty} \mu_i^k$. Notice that, by Parseval's identity, 
    $$\|z_{uv}\|^2 = (q_+^{\top}z_{uv})^2 + (q_-^{\top}z_{uv})^2 + \sum_{i=2}^{n-1} (q_i^{\top}z_{uv})^2 = (q_-^{\top}z_{uv})^2 + \sum_{i=2}^{n-1} (q_i^{\top}z_{uv})^2,$$
    since $z_{uv} \perp q_+$. Therefore, if $\overline{\mu} = \max_{i=2,\dots,n-1} |\mu_i|$ then
    $$|T_r(u,v)| \leq \frac{\overline{\mu}^{r+1}}{1 - \overline{\mu}} \sum_{i=2}^{n-1} (q_i^{\top}z_{uv})^2 \leq \frac{\overline{\mu}^{r+1}}{1 - \overline{\mu}} \|z_{uv}\|^2 = \left(\frac{1}{d_u} + \frac{1}{d_v}\right)\frac{\overline{\mu}^{r+1}}{1 - \overline{\mu}}.$$
    We now combine this with the proposition above:
    $$R(u,v) = \frac{\gamma_{uv}}{2} + \sum_{k=0}^{r} \left(\frac{(\hat{A}^k)_{uu}}{d_u} + \frac{(\hat{A}^k)_{vv}}{d_v} - \frac{2(\hat{A}^k)_{uv}}{\sqrt{d_ud_v}} - (-1)^k\gamma_{uv}\right) + T_r(u,v),$$
    which yields
    $$\sum_{k=0}^r (\hat{A}^k)_{uv} = \frac{\sqrt{d_ud_v}}{2}\left(\sum_{k=0}^{r} \left(\frac{(\hat{A}^k)_{uu}}{d_u} + \frac{(\hat{A}^k)_{vv}}{d_v} - (-1)^k\gamma_{uv}\right) + \frac{\gamma_{uv}}{2} - R(u,v) + T_r(u,v)\right).$$
    Now, we have $\sum_{k=0}^r (-1)^k = \frac{1 + (-1)^r}{2}$ and $\sum (-1)^k\gamma_{uv} = (\frac{1}{2} - \frac{1 + (-1)^r}{2})\gamma_{uv} = \frac{(-1)^{r+1}}{2}\gamma_{uv} \leq \frac{1}{2}\gamma_{uv}$. Combining everything:
    \begin{align*}
        \sum_{k=0}^r (\hat{A}^k)_{uv} 
        &\leq \frac{\sqrt{d_ud_v}}{2}\left(\sum_{k=0}^{r} \left(\frac{(\hat{A}^k)_{uu}}{d_u} + \frac{(\hat{A}^k)_{vv}}{d_v} \right) + \frac{1}{2}\gamma_{uv} - R(u,v) + \left(\frac{1}{d_u} + \frac{1}{d_v}\right)\frac{\overline{\mu}^{r+1}}{1 - \overline{\mu}}\right) \\
        &\leq \frac{d_{\max}}{2}\left(\sum_{k=0}^{r} \left(\frac{(\hat{A}^k)_{uu}}{d_u} + \frac{(\hat{A}^k)_{vv}}{d_v} \right) + \frac{1}{2}\gamma_{uv} - R(u,v) + \frac{2}{d_{\min}}\frac{\overline{\mu}^{r+1}}{1 - \overline{\mu}}\right)
    \end{align*}

    Notice that
    $$|(\hat{A}^k)_{uu}| = |\mathbbm{1}_u^{\top} \hat{A}^k \mathbbm{1}_u| \leq \|\mathbbm{1}_u^{\top}\|_2\|\hat{A}^k\|_2\|\mathbbm{1}_u\|_2 \leq \|\hat{A}^k\|_2,$$
    and since $\|\hat{A}\|_2 = 1$ we get $\|\hat{A}^k\|_2 \leq \|\hat{A}\|_2^k = 1$. Thus $(\hat{A}^k)_{uu} = |(\hat{A}^k)_{uu}| \leq 1$ and
    $$\sum_{k=0}^{r} \left(\frac{(\hat{A}^k)_{uu}}{d_u} + \frac{(\hat{A}^k)_{vv}}{d_v} \right) \leq \sum_{k=0}^{r} \left(\frac{1}{d_u} + \frac{1}{d_v} \right) \leq \sum_{k=0}^{r} \frac{2}{d_{\min}} = \frac{2}{d_{\min}}(r+1).$$
    We conclude that
    $$\sum_{k=0}^r (\hat{A}^k)_{uv} \leq \frac{d_{\max}}{2}\left(\frac{2}{d_{\min}}\left(r+1 + \frac{\overline{\mu}^{r+1}}{1 - \overline{\mu}}\right) + \frac{1}{2}\gamma_{uv} - R(u,v)\right).$$
    Plugging this into Lemma 3.2 of \cite{black2023understanding} yields the desired inequality.
\end{proof}

\begin{remark}\label{remark:bipart_corollary}
    The corollary for the total effective resistance follows from the above theorem, where the additional term accounting for bipartiteness becomes $\frac{4|U||W|}{\operatorname{vol}(G)}$, since 
$$q_- z_{uv} = \frac{\varepsilon^{\top} D^{1/2}}{\sqrt{\operatorname{vol}(G)}}D^{-1/2}(\mathbbm{1}_u - \mathbbm{1}_v) = \frac{\varepsilon_u - \varepsilon_v}{\sqrt{\operatorname{vol}(G)}},$$ thus $\gamma_{uv} = z_{uv}^{\top}P_-z_{uv} = \frac{(\varepsilon_u - \varepsilon_v)^2}{\operatorname{vol}(G)} = \begin{cases} 0, & u,v \text{ on the same side} \\ \frac{4}{\operatorname{vol}(G)}, & \text{otherwise} \end{cases}$.
\end{remark}

\section{Experimental details}
\label{app:exp_setup}

In this section, we provide additional details on the experimental setup, such as hyperparameter configurations and computational resources. For the real-world datasets related to over-squashing, we choose $d \in \{2,3,4,5\}$, three options of GNN architectures to learn restriction maps --- GCN, GIN, and SAGE --- with $\ell_{\operatorname{GNN}} \in \{1,2,3,4,5\}$. We perform grid search for the hyperparameter configurations reported in \cref{tab:sweep-configs}.

\begin{table}[htbp]
  \centering
  \caption{Model hyperparameter configurations used across all datasets. The Node and Graph Classification columns refer to \cref{tab:performance_comparison_gcn}, the TreeNeighborsMatch column refers to \cref{fig:tree-by-depth}, and the Heterophilic Graphs column refers to \cref{table:nodeclass1}. }
  \label{tab:sweep-configs}
  \begin{tabular}{lcccc}
    \toprule
      Parameter
      & Node Class.
      & Graph Class.
      & TreeNghbrsMatch
      & Heteroph. Graphs \\
    \midrule
    sheaf dimension    & 2--5             & 2--5          & 4 & 2--5
    \\
    \# layers          & 3                & 4             & $r+1 $ & 1--5
    \\
    \# hidden channels & 128              & 64            & 32 & 32,\,64
    \\
    GNN architecture   & gcn,\,gin,\,sage & gcn,\,gin,\,sage & sage & sage
    \\
    \# of GNN layers   & 1--5             & 1--5          & r+1 & 0--3
    \\
    GNN dimension      & 128              & 64            & 32 & 32,\,64
    \\
    dropout            & 0.5              & 0.0           & 0.0 & 0.2
    \\
    left weights       & \checkmark       & \checkmark    & \checkmark & \checkmark, \crossmark 
    \\
    right weights      & \checkmark       & \checkmark    & \checkmark & \checkmark, \crossmark 
    \\
    learning rate      & $0.001$          & $0.001$       & $0.001$ & $0.02\,, 0.002$ 
    \\
    \bottomrule
  \end{tabular}
\end{table}

For the ECHO-Synth benchmarks, we reproduce the experimental setup of \cite{miglior2025can}, performing Bayesian Optimization based on a GP prior in the chosen hyperparameter space. We fix the same range for hidden dimension, learning rate and weight decay proposed by the authors, layers from 10 to 40, stalk dimension and number of layers for a fixed SAGE backbone from 2 to 6, and backbone hidden dimension from 4 to 64.




\textbf{Compute resources.} Hyperparameter search for \cref{table:nodeclass1} was performed on an internal cluster equipped with NVIDIA V100, A100, and H100 GPUs. Experiments for \cref{tab:performance_comparison_small_nodeclassification} and \cref{fig:spearman_real_world} were run on a node with an Intel Xeon CPU (12 cores @ 2.20GHz) and a single NVIDIA A100-SXM4 (80GB) GPU. The remaining experiments were carried out on a workstation with two NVIDIA GeForce RTX 4090 (24GB).

\section{Computational Complexity}\label{app:complexity}
Here we compare the time complexity of FlatNSD against the time complexity of NSD with orthogonal maps. Let $n$ be the number of vertices, $m$ the number of edges, $d$ the dimension of the stalks, $h$ the number of channels, and $c = d \times h$
dimension.

\textbf{NSD.} \citet{bodnar2022neural} shows that the time complexity of NSD with orthogonal  restrictions maps is $\mathbf{O(n(c^
2+d^3) + m(cd^2 + d^3))}$,  where $d \in \{1,2,3,4,5\}$.

Importantly, \citet{bodnar2022neural} learns orthogonal maps from a composition of Householder reflections \cite{mhammedi2017efficient}. We use the same strategy to build the orthogonal maps in FlatNSD.

\textbf{FlatNSD.} We derive the time complexity of FlatNSD as follows: we need $O(d^2n)$ to embed graph features in the sheaf stalks,  $O(mcd)$ for the message-passing, $O(ncd)$ whenever we apply $W_1$, and $O(nch)$ whenever we apply $W_2$.

Thus, the total time complexity is  $\mathbf{O(cn(d+h) + mcd)}$, where $d \in \{2,3,4,5\}$.

The computational gain of FlatNSD comes from learning a single restriction map to each node instead of two restriction maps to each edges (i.e., we have flat vector bundles instead of merely vector bundles), and from an optimized implementation based on message-passing.

\begin{table*}[h]
\centering
\scriptsize 
\setlength{\tabcolsep}{6.0pt} 

\caption{Total runtime in seconds of FlatNSD and GCN with different rewiring methods for both node and graph classification tasks. DR requires graphs to have at least 3 nodes, but some graphs from Enzymes have only 2 nodes.}
\vspace{2pt}
\label{tab:runtime_comparison_gcn}
\setlength{\tabcolsep}{5pt}
\begin{tabular}{lcccccccccc}
\toprule
Method                 & Cora    & Citeseer & Texas & Cornell & Wisconsin & Chameleon & Enzymes & IMDB & Mutag & Proteins \\
\midrule
GCN+SDRF               & 926.86  & 815.83 & 640.10 & 676.62 & 676.62 & 2087.36 & \textbf{2502.20} & \textbf{2307.62} & \textbf{710.50} & \textbf{3598.63}  \\
\midrule
GCN+FOSR               & 945.13  & 825.15 & 522.83 & 677.56 & 677.56 & 2158.64 & 2601.29 & 2319.26 & 760.12 & 5962.85  \\
\midrule
GCN+BORF               & 1082.56 & 835.79 & 775.45 & 785.78 & 727.22 & 2093.88 & 4213.78 & 6410.12 & 1216.25 & 7330.80  \\
\midrule
GCN+DR                 & 920.43  & 813.23 & 641.02 & 697.87 & 666.79 & 2148.66 & -       & 2339.10 & 730.28 & 3637.15 \\
\midrule
GCN+GTR                & 898.48  & 791.14 & 638.70 & 710.84 & 647.71 & 2117.94 & 2589.34 & 2247.56 & 736.23 & 4023.95  \\
\midrule
GCN+LASER              & 1150.83 & 820.52 & 609.81 & 746.45 & 786.21 & 2143.54 & 2633.41 & 2296.16 & 741.60 & 4370.87  \\
\midrule
GCN+GOKU               & 934.48  & 805.99 & 648.15 & 662.81 & 679.77 & 2332.12 & 3780.77 & 3807.83 & 1078.83 & 4040.95 \\
\midrule
FlatNSD$_{\text{GCN}}$ & \textbf{399.96}  & \textbf{419.63} & \textbf{341.45} & \textbf{338.67} & \textbf{344.83} & \textbf{398.92} & 4093.99  & 6674.47 & 1309.02 & 8031.02 \\
\bottomrule
\end{tabular}
\end{table*}

\section{Runtime Comparison}\label{sec:runtime}
We provide Table \ref{tab:runtime_comparison_gcn} to compare the total runtime of FlatNSD against GCN coupled with seven graph rewiring strategies. On node classification benchmarks, FlatNSD exhibits superior efficiency, reducing runtime by up to $5\times$ on dense datasets like Chameleon compared to SDRF. Regarding graph classification tasks (Enzymes, IMDB, Mutag, and Proteins), while FlatNSD incurs a higher computational cost due to per-graph processing overhead, its total runtime remains within the same order of magnitude as the baseline methods. 

\begin{table}[htbp]
  \centering
  \caption{Model parameters configurations used on runtime comparison experiment.}
  \label{tab:runtime-config-real}
  \begin{tabular}{lcc}
    \toprule
      Parameter
      & Node Classification
      & Graph Classification   \\
    \midrule
    sheaf dimension      & 4                    & 4             \\
    \# layers            & 2                    & 2             \\
    \# hidden channels   & 32                   & 64            \\
    GNN architecture     & GCN                  & GCN           \\
    \# of GNN layers     & 5                    & 5             \\
    GNN dimension        & 32                   & 32            \\
    dropout              & 0.5                  & 0.0           \\
    left weights         & \checkmark           & \checkmark    \\
    right weights        & \checkmark           & \checkmark    \\
    learning rate        & $0.001$              & $0.001$       \\
    \bottomrule
  \end{tabular}
\end{table}


\end{document}